\documentclass[linksfromyear,aos,preprint]{imsart}
\usepackage[OT1]{fontenc}
\usepackage{algpseudocode}
\usepackage{algorithm}
\usepackage{appendix}

\RequirePackage{amsmath,xspace,multirow,verbatim,amssymb,graphicx,subfigure}
\RequirePackage{amsthm,natbib,url}
\RequirePackage[colorlinks,citecolor=blue,urlcolor=blue]{hyperref}
\usepackage{epsfig}
\graphicspath{{figures/}}
\usepackage{epstopdf}

\newcommand{\V}[1]{\ensuremath{\boldsymbol{#1}}\xspace}

\newtheorem{theorem}{Theorem}[section]

\newtheorem{lemma}[theorem]{Lemma}

\theoremstyle{remark}

\setattribute{journal}{name}{}
\begin{document}

\begin{frontmatter}
\title{Optimization via Low-rank Approximation for Community Detection in Networks}
\runtitle{Optimization via Low-rank Approximation}


\begin{aug}
\author{\fnms{Can M.} \snm{Le},\ead[label=cl]{canle@umich.edu}}
\author{\fnms{Elizaveta} \snm{Levina},\ead[label=el]{elevina@umich.edu}}
\and
\author{\fnms{Roman} \snm{Vershynin}\ead[label=rv]{romanv@umich.edu}}

\runauthor{Le et al.}

\affiliation{University of Michigan}

\address{311 West Hall, 1085 S. University Ave.\\
  Ann Arbor, MI 48109-1107 \\
\printead{cl}
}

\address{311 West Hall, 1085 S. University Ave.\\
  Ann Arbor, MI 48109-1107 \\
\printead{el}
}

\address{2074 East Hall, 530 Church St.\\
Ann Arbor, MI  48109-1043\\
\printead{rv}
}

\end{aug}

\begin{abstract}
Community detection is one of the fundamental problems of network analysis, for which a number of methods have been proposed.  Most model-based or criteria-based methods have to solve an optimization problem over a discrete set of labels to find communities, which is computationally infeasible.
Some fast spectral algorithms have been proposed for specific methods or models, but only on a case-by-case basis.   Here we propose a general approach for maximizing a function of a network adjacency matrix over discrete labels by projecting the set of labels onto a subspace approximating the leading eigenvectors of the expected adjacency matrix. This projection onto a low-dimensional space makes the feasible set of labels much smaller and the optimization problem much easier.   We prove a general result about this method and show how to apply it to several previously proposed community detection criteria, establishing its consistency for label estimation in each case and demonstrating the fundamental connection between spectral properties of the network and various model-based approaches to community detection. Simulations and applications to real-world data are included to
demonstrate our method performs well for multiple problems over a wide range of parameters.
\end{abstract}

\begin{keyword}[class=AMS]
\kwd[Primary ]{62E10}
\kwd[; secondary ]{62G05}
\end{keyword}

\begin{keyword}
\kwd{Community detection}
\kwd{Spectral clustering}
\kwd{Stochastic block model}
\kwd{Social networks}
\end{keyword}

\end{frontmatter}

\section{Introduction}\label{Sec:Intro}

Networks are studied in a wide range of fields, including social psychology, sociology, physics, computer science, probability, and statistics.   One of the fundamental problems in network analysis, and one of the most studied, is detecting network community structure.   Community detection is the problem of inferring the latent label vector $\V{c} \in \{1, \dots, K\}^n$ for the $n$ nodes from the observed  $n\times n$ adjacency matrix $A$,  specified by $A_{ij}=1$ if there
is an edge from $i$ to $j$, and $A_{ij}=0$ otherwise. While the problem of choosing the number of communities $K$ is important, in this paper we assume $K$ is given, as does most of the existing literature.   We  focus on the undirected network case, where the matrix $A$ is symmetric.   Roughly speaking, the large recent literature on community detection in this scenario has followed one of two tracks:  fitting probabilistic models for the adjacency matrix $A$, or optimizing global criteria derived from other considerations over label assignments $\V{c}$, often via spectral approximations.

One of the simplest and most popular probabilistic models for fitting community structure is  the stochastic block model (SBM) \cite{Holland83}.  Under the SBM, the label vector $\V{c}$ is assumed to be drawn from a multinomial distribution with parameter $\V{\pi} = \{\pi_1, \dots, \pi_K\}$, where $0\leq\pi_k\leq 1$ and $\sum_{k=1}^K\pi_k = 1$.  Edges are then formed independently between every pair of nodes $(i,j)$ with probability $P_{c_i c_j}$, and the $K \times K$ matrix $P = [P_{kl}]$  controls the probability of edges within and between communities.  Thus the labels are the only node information affecting edges between nodes, and all the nodes within the same community are stochastically equivalent to each other.   This rules out the commonly encountered ``hub'' nodes, which are nodes of unusually high degrees that are connected to many members of their own community, or simply to many nodes across the network.   To address this limitation, a relaxation that allows for arbitrary expected node degrees within communities was proposed by \cite{Karrer10}:  the degree-corrected stochastic block model (DCSBM) has $P(A_{ij} = 1) = \theta_i \theta_j P_{c_i c_j}$, where $\theta_i$'s are ``degree parameters'' satisfying some identifiability constraints.  In the ``null'' case of $K = 1$, both the block model and the degree corrected block model correspond to well-studied random graph models, the Erd\"os-R\'enyi graph \citep{Erdos&Renyi1959} and the configuration model \citep{Chung&Lu2002}, respectively.    Many other network models have been proposed to capture the community structure, for example, the latent space model \citep{Hoff2002} and  the latent position cluster model \citep{Handcock2007}.    There has also been work on extensions of the SBM  which allow nodes to belong to more than one community \citep{Airoldi2008, Ball&Karrer&Newman2011, Zhang.et.al2014overlapping}.  For a more complete review of network models, see \cite{Goldenberg2010}.

Fitting models such as the stochastic block model typically involves maximizing a likelihood function over all possible label assignments, which is in principle NP-hard.   MCMC-type and variational methods have been proposed, see for example \cite{Snijders&Nowicki1997, Nowicki2001, Mariadassouetal2010}, as well as maximizing profile likelihoods by some type of greedy label-switching algorithms.  The profile likelihood was derived for the SBM by \cite{Bickel&Chen2009} and for the DCSBM by \cite{Karrer10},  but the label-switching greedy search algorithms  only scale up to a few thousand nodes.
\cite{Amini.et.al.2013} proposed a much faster pseudo-likelihood algorithm for fitting both these models, which is based on compressing $A$ into block sums and modeling them as a Poisson mixture.  Another fast algorithm for the block model based on belief propagation has been proposed by \cite{Decelle.et.al.2011}.   Both these algorithms rely heavily on the particular form of the SBM likelihood and are not easily generalizable.

The SBM likelihood is just one example of a function that can be optimized  over all possible node labels in order to perform community detection.  Many other functions have been proposed for this purpose, often not tied to a generative network model.  One of the best-known such functions is modularity \citep{Newman&Girvan2004, Newman2006}. The key idea of modularity is to compare the observed network to a null model that has no community structure.   To define this, let $e$ be an $n$-dimensional label vector, $n_k(e) = \sum_{i=1}^n I\{e_i=k\}$ the number of nodes in community $k$,
\begin{equation}\label{Eq:Omatr}
  O_{kl}(e) = \sum_{i,j = 1}^n A_{ij} I\{e_i = k, e_j = l\}
\end{equation}
the number of edges between communities $k$ and $l$, $k \neq l$, and $O_k = \sum_{l=1}^K O_{kl}$ the sum of node degrees in community $k$.  Let $d_i=\sum_{j=1}^n A_{ij}$ be the degree of node $i$, and $m=\sum_{i=1}^n d_i$ be (twice) the total number of edges in the graph.  The Newman-Girvan modularity is derived by comparing the observed number of edges within communities to the number that would be expected under the Chung-Lu model \citep{Chung&Lu2002} for the entire graph, and can be written in the form
\begin{equation}
\label{eq:NGmod}
Q_{NG}(e) = \frac{1}{2m} \sum_k ( O_{kk} - \frac{O_k^2}{m})
\end{equation}

The quantities $O_{kl}$ and $O_k$ turn out to be the key component of many community detection criteria.  The
profile likelihoods of the SBM and DCSBM discussed above can be expressed  as
\begin{align}
  Q_{BM}(e) & = \sum_{k,l = 1}^K O_{kl}\log\frac{O_{kl}}{n_k n_l} \ , \label{bmloglik} \\
  Q_{DC}(e) &  = \sum_{k,l = 1}^K O_{kl}\log\frac{O_{kl}}{O_k O_l} \ . \label{dcloglik}
\end{align}

Another example is the extraction criterion  \citep{Zhao.et.al.2011} to extract one community at a time, allowing for arbitrary structure in the remainder of the network. The main idea is to recognize that some nodes may not belong to any community, and the strength of a community should depend on ties between its members and ties to the outside world, but not on ties between non-members.  This criterion is therefore not symmetric with respect to communities, unlike the criteria previously discussed, and has the form (using slightly different notation due to lack of symmetry),
\begin{equation}\label{ComExtr}
  Q_{EX}(V) = |V||V^c|\left(\frac{O(V)}{|V|^2}-\frac{B(V)}{|V||V^c|}\right),
\end{equation}
where $V$ is the set of nodes in the community to be extracted, $V^c$ is the complement of $V$,
  $O(V) = \sum_{i,j\in V}A_{ij}$, $B(V)=\sum_{i\in V, j\in V^c} A_{ij}$.   The only known method for optimizing this criterion is through greedy label switching,  such as the tabu search algorithm \citep{Glover&Laguna1997}.

For all these methods, finding the exact solution requires optimizing a function of the adjacency matrix $A$ over all $K^n$ possible label vectors, which is an infeasible optimization problem.  In another line of work, spectral decompositions have been used in various ways to obtain approximate solutions that are much faster to compute.   One such algorithm is spectral clustering (see, for example, \cite{Ng01}), a generic clustering method which became popular for community detection.   In this context, the method has been analyzed by  \cite{Rohe2011, Chaudhuri&Chung&Tsiatas2012, Riolo&Newman2012, Lei&Rinaldo2015}, among others, while \cite{Jin2015} proposed a spectral method specifically for the DCSBM.   In spectral clustering, typically one first computes the normalized Laplacian matrix $L = D^{-1/2} A D^{-1/2}$, where $D$ is a diagonal matrix with diagonal entries being
node degrees $d_i$, though other normalizations and no normalization at all are also possible (see \cite{Sarkar&Bickel2013}
for an analysis of why normalization is beneficial).   Then the $K$ eigenvectors of the Laplacian corresponding to the first $K$ largest eigenvalues are computed, and their rows clustered using $K$-means into $K$ clusters corresponding to different labels.   It has been shown that spectral clustering performs better with further regularization, namely if a small constant is added either to $D$ \citep{Chaudhuri&Chung&Tsiatas2012, qin2013regularized} or to $A$ \cite{Amini.et.al.2013, Joseph&Yu2013, Le&Levina&Vershynin2015}.

The contribution of our paper is a new general method of optimizing a general function $f(A, e)$ (satisfying some conditions) over labels $e$.   We start by projecting the entire feasible set of labels onto a low-dimensional subspace spanned by vectors approximating the leading eigenvectors of $EA$.    Projecting the feasible set of labels onto a low-dimensional space reduces the number of possible solutions (extreme points) from exponential to polynomial, and in particular from $O(2^n)$ to $O(n)$ for the case of two communities, thus making the optimization problem much easier. This approach is distinct from spectral clustering since one can specify any objective function $f$ to be optimized (as long as it satisfies some fairly general conditions), and thus applicable to a wide range of network problems.   It is also distinct from initializing a search for the maximum of a general function with the spectral clustering solution, since even with a good initializion the feasible space

 We show how our method can be applied to maximize the likelihoods of the stochastic block model and its degree-corrected version, Newman-Girvan modularity, and community extraction, which all solve different network problems.    While spectral approximations to some specific criteria that can otherwise be only maximized by a search over labels have been obtained on a case-by-case basis \citep{Newman2006, Riolo&Newman2012, Newman2013}, ours is, to the best of our knowledge, the first general method that would apply to any function of the adjacency matrix.   In this paper, we mainly focus on the case of two communities ($K = 2$).  For methods that are run recursively, such as modularity and community extraction, this is not a restriction. For the stochastic block model, the case $K=2$ is of special interest and has received a lot of attention in the probability literature (see \cite{Mossel.et.al.2013} for recent advances).   An extension to the general case of $K > 2$ is briefly discussed in Section \ref{Subsec:MultiCom}.


The rest of the paper is organized  as follows. In Section \ref{Sec:OptViaLRApp}, we set up notation and describe our general approach to solving a class of optimization problems over label assignments via projection onto a low-dimensional subspace. In Section \ref{Sec:AppToComDet}, we show how the general method can be applied to several community detection criteria.  Section \ref{Sec:Simulation} compares numerical performance of different methods.  The proofs are given in the Appendix.

\section{A general method for optimization via low-rank approximation}\label{Sec:OptViaLRApp}

To start with, consider the problem of detection $K = 2$ communities.   Many community detection methods rely on maximizing an objective function $f(A, e) \equiv f_A(e)$ over the set of node labels $e$, which can take values in, say, $\{-1, 1\}$.
Since $A$ can be thought of as a noisy realization of $\mathbb{E}[A]$, the ``ideal'' solution corresponds to maximizing $f_{\mathbb{E}[A]}(e)$ instead of maximizing $f_A(e)$. For a natural class of functions $f$ described below, $f_{\mathbb{E}[A]}(e)$ is essentially a function over the set of projections of labels $e$ onto the subspace spanned by eigenvectors of $\mathbb{E}[A]$ and possibly some other constant vectors.  In many cases $\mathbb{E}[A]$ is a low-rank matrix, which makes $f_{\mathbb{E}[A]}(e)$ a function of only a few variables. It is then much easier to investigate the behavior of $f_{\mathbb{E}[A]}(e)$, which typically achieves its maximum on the set of extreme points of the convex hull generated by the projection of the label set $e$.    Further, most of the $2^n$ possible label assignments $e$ become interior points after the projection, and in fact the number of extreme points is at most polynomial in $n$ (see Remark \ref{Remark:TwoComComplxty} below);  in particular, when projecting onto a two-dimensional subspace, the number of extreme points is of order $O(n)$. Therefore, we can find the maximum simply by performing an exhaustive search over the labels corresponding to the extreme points. Section~\ref{Sec:LabelEst} provides an alternative method to the exhaustive search, which is faster but approximate.

In reality, we do not know $\mathbb{E}[A]$, so we need to approximate its columns space using the data $A$ instead.    Let $U_A$ be an $m\times n$ matrix computed from $A$  such that the row space of $U_A$ approximates the column space of $\mathbb{E}[A]$ (the choice of $m \times n$ rather than $n \times m$ is for notational convenience that will become apparent below).   Existing work on spectral clustering gives us multiple option for how to compute this matrix, e.g., using the eigenvectors of $A$ itself, of its Laplacian, or of their various regularizations  -- see Section~\ref{Subsec:ApproxAdMat} for further discussion of this issue.  The algoritm works as follows:

\begin{enumerate}
  \item Compute the approximation $U_A$ from $A$.
  \item Find the labels $e$ associated with the extreme points of the projection $U_A[-1,1]^n$.
  \item Find the maximum of $f_A(e)$ by performing an exhaustive search over the set of labels found in step 2.
\end{enumerate}
Note that the first step of replacing eigenvectors of $\mathbb{E}[A]$ with certain vectors computed from $A$ is very similar to spectral clustering.   Like in spectral clustering, the output of the algorithm does not change if we replace $U_A$ with $U_A R$ for any orthogonal matrix $R$.   However, this is where the similarity ends, because instead of following the dimension reduction by an ad-hoc clustering algorithm like $K$-means, we maximize the original objective function.   The problem is made feasible by reducing the set of labels over which to maximize, to a particular subset found by taking into account the specific behavior of $f_{\mathbb{E}[A]}(e)$ and $f_A(e)$.

While our goal in the context of community detection is to compare $f_A(e)$ to $f_{\mathbb{E}[A]}(e)$, the results and the algorithm in this section apply in a general settingwhere $A$ may be any deterministic symmetric matrix.   To emphasize this generality, we write all the results in this section for a generic matrix $A$ and a generic low-rank matrix $B$, even though we will later apply them to the adjacency matrix $A$ and $B = \mathbb{E}[A]$.

Let $A$ and $B$ be $n\times n$ symmetric matrices with entries bounded by an absolute constant, and assume $B$ has rank $m\ll n$.   Assume that $f_A(e)$ has the general form 
\begin{equation}\label{Eq:GenFuncType}
  f_A(e)=\sum_{j=1}^\kappa g_j(h_{A,j}(e)),
\end{equation}
where $g_j$ are scalar functions on $\mathbb{R}$ and $h_{A,j}(e)$ are quadratic forms of $A$ and $e$, namely
\begin{equation}\label{Eq:QuadrFun}
  h_{A,j}(e) = (e+s_{j1})^T A (e+s_{j2}).
\end{equation}
Here $\kappa$ is a fixed number, $s_{j1}$ and $s_{j2}$ are constant vectors in $\{-1,1\}^n$.
Note that by \eqref{eq: Okl explicit form}, the number of edges between communities has the form \eqref{Eq:QuadrFun},
and by \eqref{eq: QDCBM explicit form}, the log-likelihood of the degree-corrected block model $Q_{DC}$ is a special case of \eqref{Eq:GenFuncType} with $g_j(x)=\pm x\log x$, $x>0$.
We similarly define $f_B$ and $h_{B,j}$,  by replacing $A$ with $B$ in \eqref{Eq:GenFuncType} and \eqref{Eq:QuadrFun}.
By allowing $e$ to take values on the cube $[-1,1]^n$, we can treat $h$ and $f$ as functions over $[-1,1]^n$.

Let $U_B$ be the $m\times n$ matrix whose rows are the $m$ leading eigenvectors of $B$. For any $e\in[-1,1]^n$, $U_Ae$ and $U_Be$ are the coordinates of the projections of $e$ onto the row spaces of
$U_A$ and $U_B$, respectively.   Since $h_{B,j}$ are quadratic forms of $B$ and $e$ and $B$ is of rank $m$, $h_{B,j}$'s depend on $e$ through $U_Be$ only, and therefore $f_B$ also depends on $e$ only through $U_B e$. In a slight abuse of notation, we also use $h_{B,j}$ and $f_B$ to denote the corresponding induced functions on $U_B[-1,1]^n$. 

Let $\mathcal{E}_A$ and $\mathcal{E}_B$ denote the subsets of labels $e\in\{-1,1\}^n$ corresponding to the sets of extreme points of $U_A[-1,1]^n$ and $U_B[-1,1]^n$, respectively. The output of our algorithm is
\begin{equation}\label{Eq:GenMethOutp}
  e^*=\mathrm{argmax}\big\{f_A(e),e\in \mathcal{E}_A\big\}.
\end{equation}


Our goal is to get a bound on the difference between the maxima of $f_A$ and $f_B$ that can be expressed through some measure of difference between $A$ and $B$ themselves.  In order to do this, we make the following assumptions.
\begin{description}
  \item[($1$)] Functions $g_j$ are continuously differentiable and there exists $M_1>0$ such that
  $|g_j^\prime(t)|\leq M_1\log(t+2)$ for $t\geq 0$.
  \item[($2$)] Function $f_B$ is convex on $U_B[-1,1]^n$.
\end{description}
Assumption (1) essentially means that Lipschitz constants of $g_j$ do not grow faster than $\log(t+2)$. The convexity of $f_B$ in assumption (2) ensures that $f_B$ achieves its maximum on $U_B\mathcal{E}_B$. In some cases (see Section~\ref{Sec:AppToComDet}), the convexity of $f_B$ can be replaced with a weaker condition, namely the convexity along a certain direction.

Let $c\in\{-1,1\}^n$ be the maximizer of $f_B$ over the set of label vectors $\{-1,1\}^n$. As a function on $U_B[-1,1]^n$, $f_B$ achieves its maximum at $U_B(c)$, which is an extreme point of $U_B[-1,1]^n$ by assumption (2). Lemma \ref{Lem:MaxFuncEst}  provides a upper bound for $f_A(c)-f_A(e^*)$.

Throughout the paper, we write $\|\cdot\|$ for the $l_2$ norm (i.e., Euclidean norm on vectors and the spectral norm on matrices), and $\|\cdot\|_F$ for the Frobenius norm on matrices. Note that for label vectors $e, c \in\{-1,1\}^n$, $\|e-c\|^2$ is four times the number of nodes on which $e$ and $c$ differ.

\begin{lemma}\label{Lem:MaxFuncEst}
If assumptions (1) and (2) hold then there exists a constant $M_2>0$ such that
\begin{equation}\label{Ineq:MaxFuncEst}
  f_T(c)-f_T(e^*)\leq M_2 n\log(n)\big( \|B\| \cdot \|U_A-U_B\|+ \|A-B\|\big),
\end{equation}
where $T$ is either $A$ or $B$.
\end{lemma}
The proof of Lemma~\ref{Lem:MaxFuncEst} is given in Appendix~\ref{AppendixA}.
To get a bound on $\|c-e^*\|$, we need further assumptions on $B$ and $f_B$.
\begin{description}
  \item[($3$)] There exists $M_3>0$ such that for any $e\in\{-1,1\}^n$, $$\|c-e\|^2\leq M_3\sqrt{n}\|U_B(c)-U_B(e)\|.$$
  \item[($4$)] There exists $M_4>0$ such that for any $x\in U_B[-1,1]^n$
  $$\frac{f_B(U_B (c))-f_B(x)}{\|U_B(c)-x\|}\geq\frac{\max f_B-\min f_B}{M_4\sqrt{n}}.$$
\end{description}

Assumption (3) rules out the existence of multiple label vectors with the same projection $U_B(c)$. Assumption (4) implies that the slope of the line connecting two points on the graph of $f_B$ at $U_B(c)$ and at any $x\in U_B[-1,1]^n$ is bounded from below. Thus, if $f_B(x)$ is close to $f_B(U_B(c))$ then $x$ is also close to $U_B(c)$. These assumptions are satisfied for all functions considered in Section~\ref{Sec:AppToComDet}.

\begin{theorem}\label{Thm:GenMetdCons}
If assumptions (1)--(4) hold, then there exists a constant $M_5$ such that
$$\frac{1}{n} \|e^*-c\|^2 \leq \frac{M_5 n \log n\big( \|B\| \cdot \|U_A-U_B\|+ \|A-B\|\big)}{\max f_B-\min f_B}.$$
\end{theorem}

Theorem~\ref{Thm:GenMetdCons} follows directly from Lemma~\ref{Lem:MaxFuncEst} and Assumptions (3) and (4). When $A$ is a random matrix,  $B=\mathbb{E}[A]$,  and $U_A$ contains the leading eigenvectors of $A$, a standard bound on $\|A-B\|$ can be applied (see Lemma~\ref{Lem:OlivRest}), which in turn yields a bound on $\|U_A-U_B\|$ by the Davis-Kahan Theorem. Under certain conditions, the upper bound in Theorem~\ref{Thm:GenMetdCons} is of order $o(n)$ (see Section~\ref{Sec:AppToComDet}), which shows consistency of $e^*$ as an estimator of $c$ (i.e., the fraction of mislabeled nodes goes to 0 as $n \rightarrow \infty$).

\subsection{The choice of low rank approximation}\label{Subsec:ApproxAdMat}
An important step of our method is replacing the ``population'' space $U_B$ with the ``data'' approximation $U_A$. As a motivating example, consider the case of the SBM,  with $A$ the network adjacency matrix and $B=\mathbb{E}[A]$.
When the network is relatively dense, eigenvectors of $A$ are good estimates of the eigenvectors of $B=\mathbb{E}[A]$ (see \cite{O'Rourke.et.al.2013} and \cite{Lei&Rinaldo2015} for recent improved error bounds).  Thus, $U_A$ can just be taken to be the leading eigenvectors of $A$.
However, when the network is sparse, this is not necessarily the best choice, since the leading eigenvectors of $A$ tend to localize around high degree nodes, while leading eigenvectors of the Laplacian of $A$ tend to localize around small connected components \cite{Mihail&Papadimitriou2002, Chaudhuri&Chung&Tsiatas2012, qin2013regularized, Le&Levina&Vershynin2015}.  This can be avoided by regularizing the Laplacian in some form;   we follow the algorithm of \cite{Amini.et.al.2013}; see also \cite{Joseph&Yu2013, Le&Levina&Vershynin2015} for theoretical analysis.    This works for both dense and sparse networks.

The regularization works as follows.  We first add a small constant $\tau$ to each entry of $A$, and then approximate $U_B$ through the Laplacian of $A+\tau\mathbf{1}\mathbf{1}^T$ as follows.
Let $D_\tau$ be the diagonal matrix whose diagonal entries are sums of entries of columns of $A+\tau\mathbf{1}\mathbf{1}^T$, $L_{\tau} = D_\tau^{-1/2} (A+\tau\mathbf{1}\mathbf{1}^T) D_\tau^{-1/2}$, and $u_i$ be leading eigenvectors of $L_{\tau}$, $1\leq i\leq K$.    Since $A+\tau\mathbf{1}\mathbf{1}^T = D_\tau ^{1/2}L_{\tau}D^{1/2}_\tau$, we set the appoximation $U_A$ the be the basis of the span of $\{D^{1/2}u_i:1\leq i\leq K\}$.  Following \cite{Amini.et.al.2013},  we set $\tau=\varepsilon(\lambda_n/n)$, where $\lambda_n$ is the node expected degree of the network and $\varepsilon\in (0,1)$ is a constant which has little impact on the performance \cite{Amini.et.al.2013}.

\subsection{Computational complexity}\label{Remark:TwoComComplxty}

Since we propose an exhaustive search over the projected set of extreme points, the computational feasibility of this is a concern.    A projection of the unit cube  $U_A[-1,1]^n$ is the Minkowski sum of $n$ segments in $\mathbb{R}^m$, which, by \cite{Gritzmann1993}, implies that it has $O(n^{m-1})$ vertices of $U_A[-1,1]^n$ and they can be found in $O(n^m)$ arithmetic operations.     When $m=2$, which is the primary focus of our paper,  there exists an algorithm that can find the vertices of $U_A[-1,1]^n$ in $O(n\log n)$ arithmetic operations \cite{Gritzmann1993}.     Informally, the algorithm first sorts the angles between the $x$-axis and column vectors of $U_A$ and $-U_A$. It then starts at a vertex of $U_A[-1,1]^n$ with the smallest $y$-coordinate, and based on the order of the angles,  finds neighbor vertices of $U_A[-1,1]^n$ in a counter-clockwise order.  If the angles are distinct (which occurs with high probability), moving from one vertex to the next causes exactly one entry of the corresponding label vector to change the sign, and therefore the values of $h_{A,j}(e)$ in \eqref{Eq:QuadrFun} can be updated efficiently. In particular, if $A$ is the adjacency matrix of a network with average degree $\lambda_n$, then on avarage, each update takes $O(\lambda_n)$ arithmetic operations, and given $U_A$, it only takes $O(n\lambda_n \log n)$ arithmetic operations to find $e^*$ in \eqref{Eq:GenMethOutp}.    Thus the computational complexity of this search for two communities is not at all prohibitive -- compare to the computational complexity of finding $U_A$ itself, which is at least $O(n\lambda_n \log n)$ for $m=2$.

\subsection{Extension to more than two communities}\label{Subsec:MultiCom}

Let $K$ be the number of communities and $S$ be an $n\times K$ label matrix: for $1\leq i\leq n$, if node $i$ belongs to community $k$ then $S_{ik}=1$ and $S_{il}=0$ for all $l\neq k$. The numbers of edges between communities defined by \eqref{Eq:Omatr} are entries of $S^TAS$. Let $B = \sum_{i=1}^K \rho_i\bar{u}_i \bar{u}_i^T$ define the eigendecomposition of $B$. The population version of $S^TAS$ is
\begin{equation*}
  S^TBS = S^T\left(\sum_{j=1}^K \rho_j \bar{u}_j \bar{u}_j^T\right)S = \sum_{j=1}^K\rho_j \left(S^T\bar{u}_j\right)\left(S^T\bar{u}_j\right)^T.
\end{equation*}
Let $U_B$ be the $K\times n$ matrix whose rows are $\bar{u}_j^T$. Then $S^TBS$ is a function of $U_B S$.   We approximate $U_B$ by $U_A$ described in Section~\ref{Subsec:ApproxAdMat}. Let $\tilde{S}$ be the the first $K-1$ columns of $S$. Note that the rows of $S$ sum to one, therefore $U_AS$ can be recovered from $U_A\tilde{S}$. Now relax the entries of $\tilde{S}$ to take values in $[0,1]$, with the row sums of at most one. For $1\leq i\leq n$ and $1\leq j\leq K-1$, denote by $V_{ij}$ the $K\times (K-1)$ matrix such that the $j$-th column of $V_{ij}$ is the $i$-th column of $U_A$ and all other columns are zero.
Then
\begin{equation*}
  U_A\tilde{S} = \sum_{i=1}^n \sum_{j=1}^{K-1} \tilde{S}_{ij} V_{ij}.
\end{equation*}
Since $\sum_{j=1}^{K-1} \tilde{S}_{ij}\leq 1$, $\sum_{j=1}^{K-1} \tilde{S}_{ij} V_{ij}$ is a convex set in $\mathbb{R}^{K\times (K-1)}$, isomorphic to a $K-1$ simplex. Thus, $U_A\tilde{S}$ is a Minkowski sum of $n$ convex sets in $\mathbb{R}^{K\times(K-1)}$. Similar to the case $K=2$, we can first find the set of label matrices $\tilde{S}$ corresponding to the extreme points of $U_A\tilde{S}$ and then perform the exhaustive search over that set.

A bound on the number of vertices of $U_A\tilde{S}$ and a polynomial algorithm to find them are derived by \cite{Gritzmann1993}.   If $d=K(K-1)$, then the number of vertices of $U_A\tilde{S}$ is at most $O\left(n^{(d-1)} K^{2(d-1)}\right)$, and they can be found in $O\left(n^d K^{(2d-1)}\right)$ arithmetic operations. presents An implementation of the reverse-search algorithm of \cite{Fukuda2004} for computing the Minkowski sum of polytopes was presented in \cite{Weibel2010} , who  showed that the algorithm can be parallelized efficiently.   We do not pursue these improvements here, since our main focus in this paper is the case $K = 2$.

\section{Applications to community detection}\label{Sec:AppToComDet}
Here we apply the general results from Section~\ref{Sec:OptViaLRApp} to a network adjacency matrix $A$, $B=\mathbb{E}[A]$, and functions corresponding to several popular community detection criteria.
Our goal is to show that our maximization method gets an estimate close to  the true label vector $c$, which is the maximizer of the corresponding function with $B = \mathbb{E}[A]$ plugged in for $A$. We focus on the case of two communities and use $m=2$ for the low rank approximation.

Recall the quantities $O_{11}$, $O_{22}$, and $O_{12}$  defined in \eqref{Eq:Omatr}, which are used by all the criteria we consider.  They are quadratic forms of $A$ and $e$ and can be written as
\begin{eqnarray}\label{eq: Okl explicit form}
  O_{11}(e) &=& \frac{1}{4}(\mathbf{1}+e)^T A (\mathbf{1}+e), \ \ \
  O_{22}(e)  = \frac{1}{4}(\mathbf{1}-e)^T A (\mathbf{1}-e), \\
  \nonumber O_{12}(e) &=& \frac{1}{4}(\mathbf{1}+e)^T A (\mathbf{1}-e), \label{Oofe}
\end{eqnarray}
where $\mathbf{1}$ is the all-ones vector.

\subsection{Maximizing the likelihood of the degree-corrected stochastic block model}\label{SubSec:MaxLogLikDCBM}
When a network has two communities, \eqref{dcloglik} takes the form
\begin{eqnarray}\label{eq: QDCBM explicit form}
   Q_{DC}(e) &=& O_{11} \log O_{11} + O_{22} \log O_{22} + 2 O_{12} \log O_{12}\\
   \nonumber &-& 2 O_1 \log O_1 - 2 O_2 \log O_2.
\label{Qdcbm}
\end{eqnarray}
Thus, $Q_{DC}$ has the form defined by \eqref{Eq:GenFuncType}.

For simplicity, instead of drawing $c$ from a multinomial distribution with parameter $\pi=(\pi_1,\pi_2)$, we fix the true label vector by assigning the first $\bar{n}_1=n\pi_1$ nodes to community 1 and the remaining $\bar{n}_2=n\pi_2$ nodes to community 2.
Let $r$ be the out-in probability ratio, and
\begin{equation}
P=\lambda_n\left(
     \begin{array}{cc}
       1 & r \\
       r & \omega \\
     \end{array}
   \right) \
\label{eq:probmatrix}
\end{equation}
be the probability matrix.
We assume that the node degree parameters $\theta_i$ are an i.i.d.\ sample from a distribution with $\mathbb{E}[\theta_i]=1$ and $1/\xi\leq \theta_i\leq \xi$ for some constant $\xi\geq 1$.
The adjacency matrix $A$ is symmetric and for $i > j$ has independent entries generated by $A_{ij}=\mathrm{Bernoulli}(\theta_i\theta_jP_{c_ic_j})$.
Throughout the paper, we let $\lambda_n$ depend on $n$, and fix $r$, $\omega$, $\pi$, and $\xi$. Since $\lambda_n$ and the network expected node degree are of the same order, in a slight abuse of notation, we also denote by $\lambda_n$ the network expected node degree.

Theorem~\ref{Thm:DCBMCons} establishes consistency of our method in this setting.
\begin{theorem}\label{Thm:DCBMCons}
Let $A$ be the adjacency matrix generated from the DCSBM with $\lambda_n$ growing at least as $\log^2 n $ as $n\rightarrow\infty$.
Let $U_A$ be an approximation of $U_{\mathbb{E}[A]}$, and $e^*$  the label vector defined by \eqref{Eq:GenMethOutp} with $f_A=Q_{DC}$.
Then for any $\delta\in(0,1)$, there exists a constant $M=M(r,\omega,\pi,\xi,\delta)>0$ such that with probability at least $1-\delta$, we have
$$\frac{1}{n} \|c-e^*\|^2\leq M \log n \left( \lambda_n^{-1/2} + \|U_A-U_{\mathbb{E}[A]}\| \right).$$
In particular, if $U_A$ is a matrix whose row vectors are leading eignvectors of $A$, then the fraction of mis-clustered nodes is bounded by $M\log n/\sqrt{\lambda_n}$.
\end{theorem}
Note that assumption ($2$) is difficult to check for $Q_{DC}$ but a weaker version, namely convexity along a certain direction, is sufficient for proving Theorem~\ref{Thm:DCBMCons}. The proof of Theorem~\ref{Thm:DCBMCons} consists of checking assumptions ($1$), ($3$), ($4$), and a weaker version of assumption ($2$). For details, see Appendix~\ref{AppendixB:DCBM}.

\subsection{Maximizing the likelihood of the stochastic block model}\label{SubSec:MaxLogLikBM}
While the regular SBM is a special case of DCSBM when $\theta_i=1$ for all $i$, its likelihood is different and thus maximizing it gives a different solution.
With two communities, \eqref{bmloglik} admits the form
\begin{eqnarray}\label{Eq:LogLikBM}
   Q_{BM}(e) = Q_{DC}(e)+2O_1\log\frac{O_1}{n_1} + 2 O_2\log\frac{O_2}{n_2},\nonumber
\end{eqnarray}
where $n_1=n_1(e)$ and $n_2=n_2(e)$ are the numbers of nodes in two communities and can be written as
\begin{equation}\label{Eq:Defn1n2}
  n_1 = \frac{1}{2}(\mathbf{1}+e)^T \mathbf{1}=\frac{1}{2}(n+e^T\mathbf{1}), \ \
  n_2 = \frac{1}{2}(\mathbf{1}-e)^T \mathbf{1}=\frac{1}{2}(n-e^T\mathbf{1}).
\end{equation}
\begin{theorem}\label{Thm:BMCons}
Let $A$ be the adjacency matrix generated from the SBM with $\lambda_n$ growing at least as $\log^2 n$ as $n\rightarrow\infty$.
Let $U_A$ be an approximation of $U_{\mathbb{E}[A]}$, and $e^*$  the label vector defined by \eqref{Eq:GenMethOutp} with $f_A=Q_{BM}$.
Then for any $\delta\in(0,1)$, there exists a constant $M=M(r,\omega,\pi,\xi,\delta)>0$ such that with probability at least $1-n^{-\delta}$, we have
$$\frac{1}{n} \|c-e^*\|^2\leq M  \log n  \left( \lambda_n^{-1/2} + \|U_A-U_{\mathbb{E}[A]}\| \right).$$
In particular, if $U_A$ is a matrix whose row vectors are leading eignvectors of $A$, then the fraction of mis-clustered nodes is bounded by $M\log n/\sqrt{\lambda_n}$.
\end{theorem}

Note that $Q_{BM}$ does not have the exact form of \eqref{Eq:GenFuncType} but a small modification shows that Lemma~\ref{Lem:MaxFuncEst} still holds for $Q_{BM}$. Also, assumption ($2$) is difficult to check for $Q_{BM}$ but again a weaker condition of convexity along a certain direction is sufficient for proving Theorem~\ref{Thm:BMCons}. The proof of Theorem~\ref{Thm:BMCons} consists of showing the analog of Lemma~\ref{Lem:MaxFuncEst}, checking assumptions ($3$), ($4$), and a weaker version of assumption ($2$). For details, see Appendix~\ref{AppendixB:BM}.

\subsection{Maximizing the Newman--Girvan modularity}\label{Subsec:NGMod}
When a network has two communities, up to a constant factor  the modularity  \eqref{eq:NGmod} takes the form
\begin{equation*}
  Q_{NG}(e) = O_{11} + O_{22}-\frac{O_1^2+O_2^2}{O_1+O_2}
   = \frac{2O_1 O_2}{O_1+O_2} - 2O_{12}.
\end{equation*}
Again,  $Q_{NG}$ does not have the exact form \eqref{Eq:GenFuncType}, but with a small modification, the argument used for proving Lemma~\ref{Lem:MaxFuncEst} and Theorem~\ref{Thm:GenMetdCons} still holds for $Q_{NG}$ under the regular SBM.

\begin{theorem}\label{Thm:NGCons}
Let $A$ be the adjacency matrix generated from the SBM with $\lambda_n$ growing at least as $\log n$ as $n\rightarrow\infty$.
Let $U_A$ be an approximation of $U_{\mathbb{E}[A]}$,  and $e^*$  the label vector defined by \eqref{Eq:GenMethOutp} with $f_A=Q_{NG}$.
Then for any $\delta\in(0,1)$, there exists a constant $M=M(r,\omega,\pi,\xi,\delta)>0$ such that with probability at least $1-n^{-\delta}$, we have
$$\frac{1}{n} \|c-e^*\|^2\leq M  \left( \lambda_n^{-1/2} + \|U_A-U_{\mathbb{E}[A]}\| \right).$$
In particular, if $U_A$ is a matrix whose row vectors are leading eignvectors of $A$, then the fraction of mis-clustered nodes is bounded by $M/\sqrt{\lambda_n}$.
\end{theorem}

It is easy to see that $Q_{NG}$ is Lipschitz with respect to $O_1$, $O_2$, and $O_{12}$, which is stronger than assumption ($1$) and ensures the proof of Lemma~\ref{Lem:MaxFuncEst} goes through. The proof of Theorem~\ref{Thm:NGCons} consists of checking assumptions ($2$), ($3$), ($4$), and the Lipschitz condition for $Q_{NG}$. For details, see 
Appendix~\ref{AppendixB:NG}.

\subsection{Maximizing the community extraction criterion}\label{SubSec:MaxComExtrCrn}
Identifying the community $V$ to be extracted with a label vector $e$, the criterion \eqref{ComExtr} can be written as
$$Q_{EX}(e) = \frac{n_2}{n_1} O_{11} - O_{12},$$
where $n_1, n_2$ are defined by \eqref{Eq:Defn1n2}.   Once again $Q_{EX}$  does not have the exact form \eqref{Eq:GenFuncType}, but with small modifications of the proof, Lemma~\ref{Lem:MaxFuncEst} and Theorem~\ref{Thm:GenMetdCons} still hold for $Q_{EX}$.

\begin{theorem}\label{Thm:ComExtrCons}
Let $A$ be the adjacency matrix generated from the SBM with the probability matrix \eqref{eq:probmatrix}, $\omega = r$,  and $\lambda_n$ growing at least as $\log n$ as $n\rightarrow\infty$.
Let $U_A$ be an approximation of $U_{\mathbb{E}[A]}$,  and $e^*$  the label vector defined by \eqref{Eq:GenMethOutp} with $f_A=Q_{EX}$.
Then for any $\delta\in(0,1)$, there exists a constant $M=M(r,\omega,\pi,\xi,\delta)>0$ such that with probability at least $1-n^{-\delta}$, we have
$$\frac{1}{n} \|c-e^*\|^2\leq M  \left( \lambda_n^{-1/2} + \|U_A-U_{\mathbb{E}[A]}\| \right).$$
In particular, if $U_A$ is a matrix whose row vectors are leading eignvectors of $A$, then the fraction of mis-clustered nodes is bounded by $M/\sqrt{\lambda_n}$.
\end{theorem}

The proof of Theorem~\ref{Thm:ComExtrCons} consists of verifying a version of Lemma~\ref{Lem:MaxFuncEst} and assumptions ($2$), ($3$), and ($4$), and is included in 
Appendix~\ref{AppendixB:EXTR}.

\subsection{An alternative to exhaustive search}\label{Sec:LabelEst}
While the projected feasible space is much smaller than the original space, we may still want to avoid the exhaustive search for $e^*$ in \eqref{Eq:GenMethOutp}.   The geometry of the projection of the cube can be used to derive an approximation to $e^*$ that can be computed without a search.

\begin{figure}[!ht]
  \centering
  \includegraphics[trim=50 30 40 20,clip,width=0.9\textwidth]{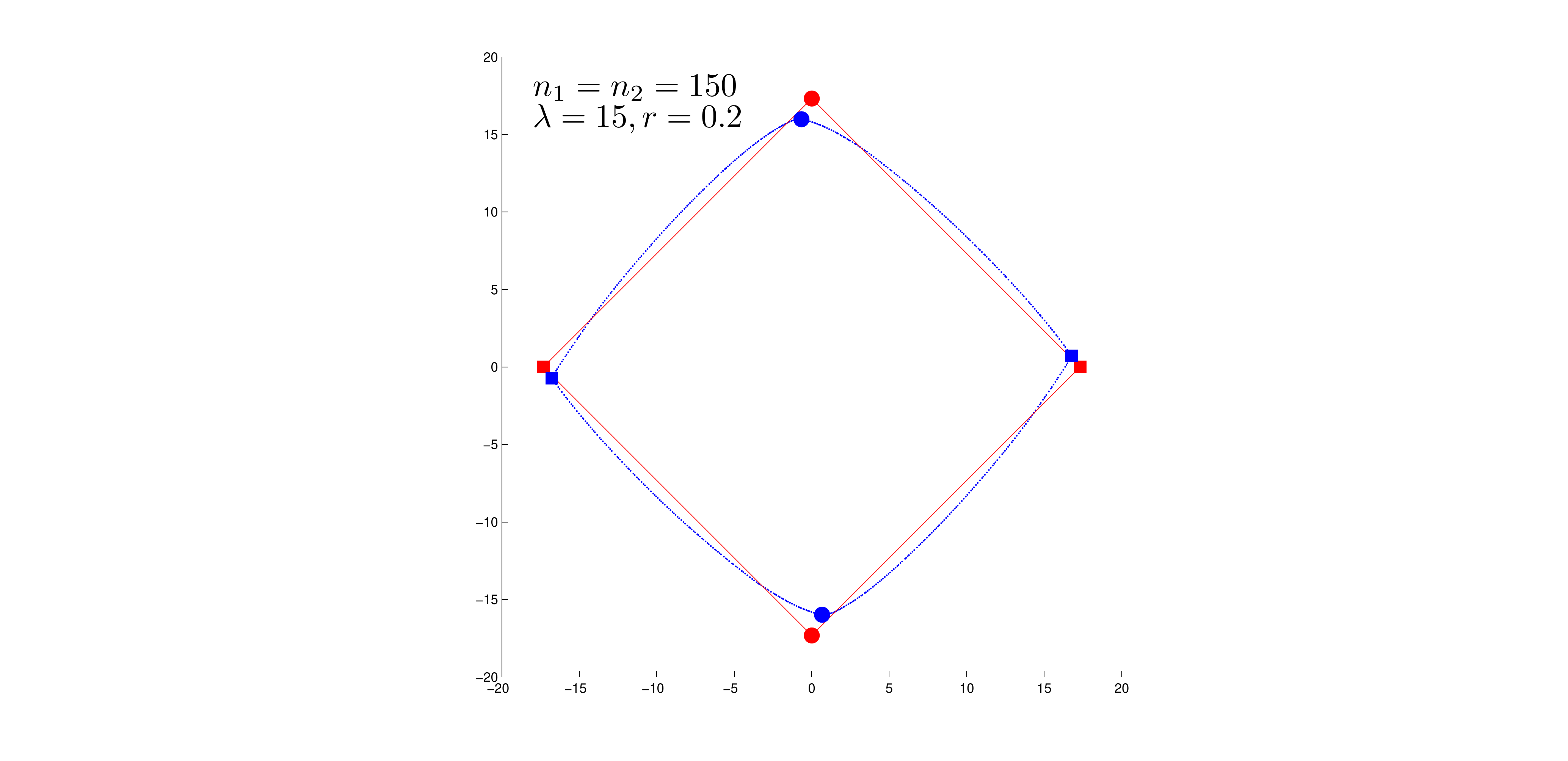}\\
  \caption{The projection of the cube $[-1,1]^n$ onto two-dimensional subspace. Blue corresponds to the projection onto eigenvectors of $A$, and red onto the eigenvectors of $\mathbb{E}[A]$. The red contour is the boundary of $U_{\mathbb{E}[A]}[-1,1]^n$; the blue dots are the extreme points of $U_A[-1,1]^n$. Circles (at the corners) are $\pm$ projections of the true label vector;  squares are $\pm$ projections of the vector of all 1s.
}
  \label{Fig:CubeProj}
\end{figure}

Recall that $U_{\mathbb{E}[A]}$ is an $2\times n$ matrix whose rows are the leading eigenvectors of $\mathbb{E}[A]$, and $U_A$ approximates $U_{\mathbb{E}}[A]$. For SBM,  it is easy to see that $U_{\mathbb{E}[A]}[-1,1]^n$, the projection of the unit cube onto the two leading eigenvectors of $U_{\mathbb{E}[A]}$, is a parallelogram with vertices $\{\pm U_{\mathbb{E}[A]}\mathbf{1},\pm U_{\mathbb{E}[A]} c\}$, where $\mathbf{1} \in \mathbb{R}^n$ is a vector of all 1s  (see Lemma~\ref{Lem:BMCubeProj} in the supplement).   We can then expect the projection $U_A[-1,1]^n$ to look somewhat similar -- see the illustration in Figure~\ref{Fig:CubeProj}.
Note that $\pm U_{\mathbb{E}[A]} c$ are the farthest points from the line connecting the other two vertices, $U_{\mathbb{E}[A]}\mathbf{1}$ and $-U_{\mathbb{E}[A]}\mathbf{1}$.
Motivated by this observation, we can estimate $c$ by
\begin{eqnarray}\label{Eq:LabelEst}
  \hat{c} &=& \arg\max\left\{\langle U_A e,(U_A\mathbf{1})^{\perp}\rangle: e\in\{-1,1\}^n\right\} \\
  &=& \mathrm{sign}(u_1^T\mathbf{1}u_2 - u_2^T\mathbf{1} u_1),\nonumber
\end{eqnarray}
where $U_A=(u_1,u_2)^T$ and $(U_A\mathbf{1})^{\perp}$ is the unit vector perpendicular to $U_A\mathbf{1}$.

Note that $\hat{c}$ depends on $U_A$ only, not on the objective function, a property it shares with spectral clustering.
However,  $\hat{c}$ provides a deterministic estimate of the labels based on a geometric property of $U_A$,
while spectral clustering uses $K$-means, which is iterative and typically depends on a random initialization.
Using this geometric approximation allows us to avoid both the exhaustive search and the iterations and initialization of $K$-means, although it may not always be as accurate as the search.
When the community detection problem is relatively easy, we expect the geometric approximation to perform well, but when the problem becomes harder, the exhaustive search should provide better results.   This intuition is confirmed by simulations in Section~\ref{Sec:Simulation}.
Theorem~\ref{Thm:EstErrorBound} shows that $\hat{c}$ is a consistent estimator. The proof is given in Appendix~\ref{AppendixB:FormEst}.

\begin{theorem}\label{Thm:EstErrorBound}
Let $A$ be an adjacency matrix generated from the SBM with $\lambda_n$ growing at least as $\log n$ as $n\rightarrow\infty$.
Let $U_A$ be an approximation to $U_{\mathbb{E}[A]}$.
Then for any $\delta\in(0,1)$ there exists $M=M(r,\omega,\pi,\xi,\delta)>0$ such that with probability at least $1-n^{-\delta}$, we have
$$\frac{1}{n} \|\hat{c}-c\|^2 \leq M \|U_A - U_{\mathbb{E}[A]}\|^2.$$
In particular, if $U_A$ is a matrix whose row vectors are leading eignvectors of $A$,
then the fraction of mis-clustered nodes is bounded by $M/\lambda_n$.
\end{theorem}

\subsection{Theoretical comparisons}
There are several results on the consistency of recovering the true label vector under both the SBM and the DCSBM.
The balanced planted partition model $G(n,\frac{a}{n},\frac{b}{n})$, which is the simplest special case of the SBM, has received much attention recently, especially in the probability literature.
This model assumes that there are two communities with $n/2$ nodes each, and edges are formed within communities and between communities with probabilities $a/n$ and $b/n$, respectively. When $(a-b)^2\leq 2(a+b)$, no method can find the communities  \cite{Mossel.et.al.2012}.   Algorithms based on non-backtracking random walks that can recover the community structure better than random guessing if $(a-b)^2>2(a+b)$ have been proposed in \cite{Mossel&Neeman&Sly2014, Massoulie:2014}   Moreover, if $(a-b)^2/(a+b)\rightarrow\infty$ as $n\rightarrow\infty$ then the fraction of mis-clustered nodes goes to zero with high probability.
Under the model $G(n,\frac{a}{n},\frac{b}{n})$, our theoretical results require that $a+b$ grows at least as $\log n$.  This matches the requirements on the expected degree $\lambda_n$ needed for consistency in \cite{Bickel&Chen2009} for the SBM and in \cite{Zhaoetal2012} for the DCSBM.

When the expected node degree $\lambda_n$ is of order $\log n$, spectral clustering using eigenvectors of the adjacency matrix can correctly recover the communities, with fraction of mis-clustered nodes up to $O(1/\log n)$ \cite{Lei&Rinaldo2015}.
In this regime, our method for maximizing the Newman-Girvan and the community extraction criteria mis-clusters at most $O(1/\sqrt{\lambda_n})$ fraction of the nodes. For maximizing the likelihoods of the SBM and DCSBM, we require that $\lambda_n$ is of order $\log^2n$, and the fraction of mis-clustered nodes is bounded by $O(\log n/\sqrt{\lambda_n})$.
For Newman-Girvan modularity as well as the SBM likelihood, \cite{Bickel&Chen2009}  proved strong consistency (perfect recovery with high probability) under the SBM when $\lambda_n$ grows faster than $\log n$.   However, they used a label-switching algorithm for finding the maximizer, which is computationally infeasible for larger networks.    A much faster algorithm based on pseudo-likelihood was proposed by \cite{Amini.et.al.2013}, who assumed that the initial estimate of the labels (obtained in practice by regularized spectral clustering) has a certain correlation with the truth, and showed that the fraction of mis-clustered nodes for their method is $O(1/\lambda_n)$.  Recently, \cite{Le&Levina&Vershynin2015} analyzed regularized spectral clustering in the sparse regime when $\lambda_n = O(1)$, and showed that with high probability, the fraction of mis-clustered nodes is $O(\log^6 \lambda_n/\lambda_n)$. In summary, our assumptions required for consistency are similar to others in the literature even though the approximation method is fairly general.


\section{Numerical comparisons}\label{Sec:Simulation}
Here we briefly compare the empirical performance of our
extreme point projection method to several other methods for community detection, both general (spectral
clustering) and those designed specifically for optimizing a
particular community detection criterion, using both simulated networks and two real
network datasets, the political blogs and the dolphins data described in in Section~\ref{sec:data}.  Our goal in this comparison
is to show that our general method does as well as the algorithms tailored to a particular
criterion, and thus we are not trading off accuracy for generality.

For the four criteria discussed in Section~\ref{Sec:AppToComDet}, we
compare our method of maximizing the relevant criterion by exhaustive
search over the extreme points of the projection  (EP, for extreme points), the approximate version based on the
geometry of the feasible set described in Section~\ref{Sec:LabelEst}
(AEP, for approximate extreme points), and regularized spectral
clustering (SCR) proposed by
\cite{Amini.et.al.2013},  which are all general methods.   We also
include one method specific to the criterion in each comparison.   For
the SBM, we compare to the unconditional
pseudo-likelihood (UPL) and for the DCSBM, to the
conditional pseudo-likelihood (CPL), two fast and accurate methods developed specifically for these models by
\cite{Amini.et.al.2013}.    For the Newman-Girvan modularity, we compare
to the spectral algorithm of
\cite{Newman2006}, which uses the leading eigenvector of the modularity matrix (see
details in Section~\ref{SubSec:Sim:MaxNGmod}).  Finally, for community extraction we compare to the algorithm proposed
in the original paper \citep{Zhao.et.al.2011} based on greedy label
switching, as there are no faster algorithms available.

The simulated networks are generated using the parametrization of
\cite{Amini.et.al.2013}, as follows.  Throughout this section, the
number of nodes in the network is fixed at $n=300$, the number of
communities $K = 2$, and the true label
vector $c$ is fixed.    The number of replications for each setting is 100.
First, the node
degree parameters $\theta_i$ are drawn independently from
the distribution $\mathbb{P}(\Theta = 0.2)=\gamma$, and
$\mathbb{P}(\Theta = 1)=1-\gamma$.   Setting $\gamma = 0$ gives the
standard SBM, and $\gamma > 0$ gives the DCSBM, with  $1-\gamma$ the fraction of hub
nodes.  The matrix of edge probabilities $P$ is controlled by two
parameters: the out-in probability ratio $r$, which determines how
likely edges are formed within and between communities, and the weight
vector $w = (w_1, w_2)$, which determines the relative node degrees within communities.
Let
$$P_0 = \begin{bmatrix}
w_1 & r \\
r & w_2
\end{bmatrix} .
$$
The difficulty of the problem is largely controlled by $r$ and
the overall expected network degree $\lambda$.   Thus we rescale $P_0$
to control the expected degree, setting
$$P = \frac{\lambda P^0}{(n-1)(\pi^T P^0\pi)(\mathbb{E}[\Theta])^2},$$
where $\pi = n^{-1}(n_1, n_2)$, and $n_k$ is the number of
nodes in community $k$.   Finally, edges $A_{ij}$
are drawn independently from a Bernoulli distribution with
$\mathbb{P}(A_{ij} = 1) = \theta_i \theta_j P_{c_i c_j}$.


As discussed in Section \ref{Subsec:ApproxAdMat}, a good approximation to the eigenvectors of  $\mathbb{E}[A]$ is provided by the eigenvectors of the regularized Laplacian.   SCR uses these eigenvectors $u_1$, $u_2$ as input to $K$-means (computed here with the kmeans function in Matlab with 40 random initial starting points).    EP and AEP use $\{D^{1/2}u_1,D^{1/2}u_2\}$ to compute the matrix $U_A$ (see Section~\ref{Subsec:ApproxAdMat}).
To find extreme points and corresponding label vectors in
the second step of EP, we use the algorithm of \cite{Gritzmann1993}.  For $m=2$, it essentially
consists of sorting the angles of between the column vectors of $U_A$ and the $x$-axis.
In case of multiple maximizers, we break the tie by choosing the label vector whose projection
is the farthest from the line connecting the projections of $\pm\mathbf{1}$ (following
the geometric idea of Section~\ref{Sec:LabelEst}).
For CPL and UPL, following \cite{Amini.et.al.2013},
we initialize with the output of SCR and set the number of outer iterations to 20.

We measure the accuracy of all methods via the normalized mutual
information (NMI) between the label vector $c$ and its estimate $e$.
NMI takes values between 0 (random guessing) and 1 (perfect match), and
is defined by \cite{Yao03} as
$\text{NMI}(c,e) = -\sum_{i,j}R_{ij}\log\frac{R_{ij}}{R_{i+}R_{+j}}\left(\sum_{ij}R_{ij}\log R_{ij}\right)^{-1}$, where
 $R$ is the confusion matrix between $c$ and $e$, which represents a bivariate probability distribution,
and its row and column sums $R_{i+}$ and $R_{+j}$ are the corresponding marginals.

\subsection{The degree-corrected stochastic block model}\label{SubSec:Sim:MaxLogLikDCBM}

\begin{figure}[!ht]
  \centering
  \includegraphics[trim=90 30 80 35,clip,width=0.99\textwidth]{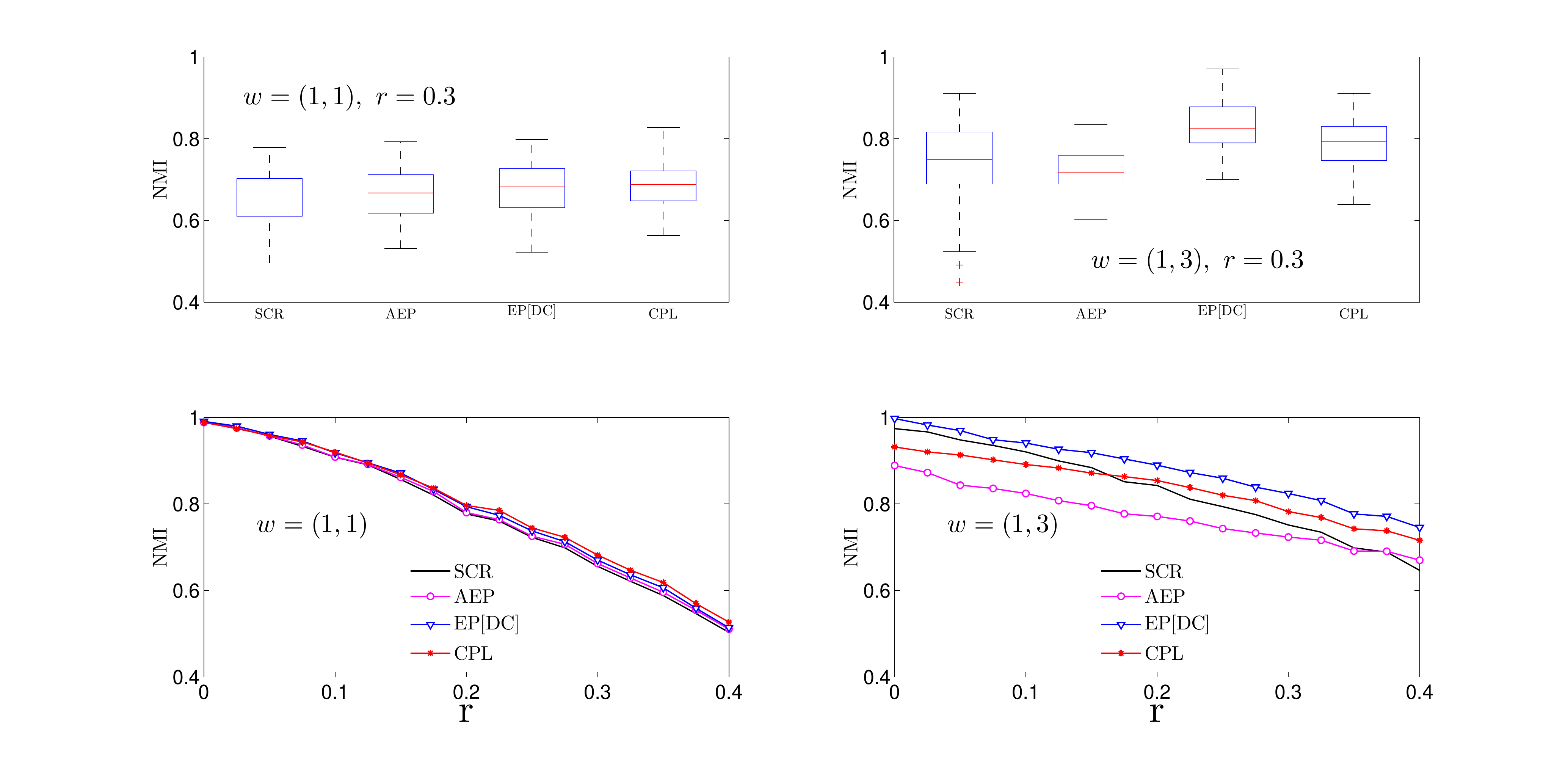}\\ 
  \caption{The degree-corrected stochastic block model.  Top row: boxplots of NMI between true and estimated
    labels.   Bottom row:  average NMI against the out-in probability ratio $r$. In all plots, $n_1=n_2=150$, $\lambda=15$, and $\gamma=0.5$.
  }\label{Fig:DCBMBoxplots}
\end{figure}

Figure~\ref{Fig:DCBMBoxplots} shows the
performance of the four methods for fitting the DCSBM under different parameter
settings. We use the notation EP[DC] to emphasize that EP
here is used to maximize the log-likelihood of DCSBM. In
this case, all methods perform similarly, with EP performing the best when
community-level degree weights are different ($w = (1, 3)$), but just slightly worse than CPL when $w = (1, 1)$.
The AEP is always somewhat worse than the exact version, especially when
$w = (1, 3)$, but overall their results are comparable.

\subsection{The stochastic block model}
Figure~\ref{Fig:BoxPlotBM} shows the performance of the four methods
for fitting the regular SBM ($\gamma=0$).    Over all, four methods provide quite similar results, as
we would hope good fitting methods will.    The performance of the appoximate method AEP is very similar to that of EP, and the model-specific UPL marginally outperforms the three general methods.
\begin{figure}[!ht]
  \centering
  \includegraphics[trim=90 30 80 35,clip,width=0.99\textwidth]{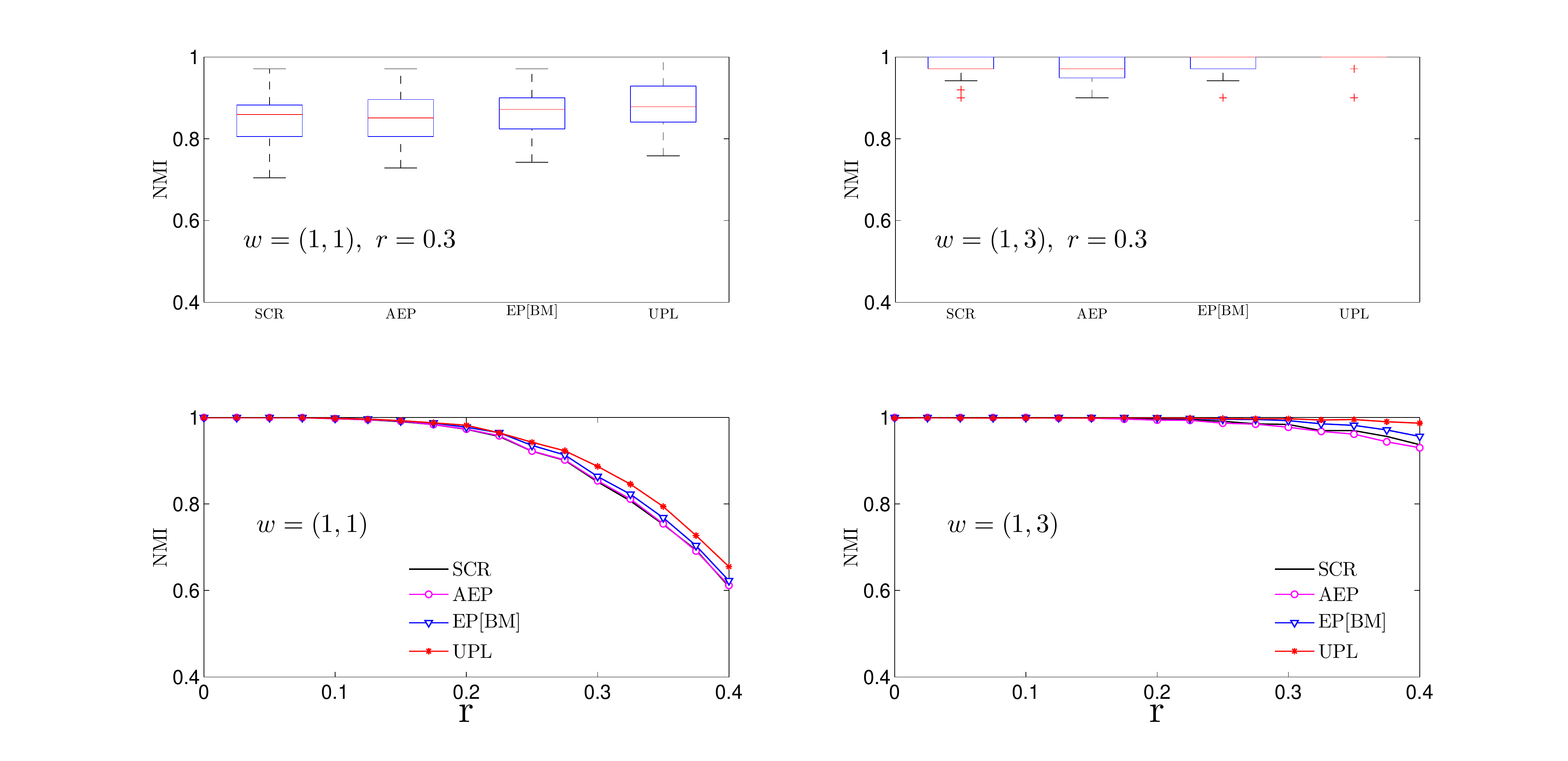}\\
  \caption{The stochastic block model.   Top row: boxplots of NMI between true and estimated
    labels.   Bottom row:  average NMI against the out-in probability
    ratio $r$.  In all plots, $n_1=n_2=150$, $\lambda=15$, and $\gamma=0$.}
  \label{Fig:BoxPlotBM}
\end{figure}

\subsection{Newman--Girvan modularity}\label{SubSec:Sim:MaxNGmod}
The modularity function $\hat{Q}_{NG}$  can be approximately maximized via a fast spectral algotithm when partitioning into two communities \cite{Newman2006}.
Let $B = A-P$ where $P_{ij} = d_id_j/m$, and write $\hat{Q}_{NG}(e)=\frac{1}{2m}e^TBe$.
The approximate solution (LES, for leading eigenvector signs) assigns
node labels according to the signs of the corresponding entries of the leading eigenvector of $B$.
For a fair comparison to other methods relying on eigenvectors, we also use the regularized
$A+\tau\mathbf{1}\mathbf{1}^T$ instead of $A$ here, since empirically
we found that it slightly improves the
performance of LES.
Figure~\ref{Fig:NewmanGirvanPlots} shows the performance of AEP,
EP[NG], and LES, when the data are generated from a regular block model
($\gamma=0$).   The two extreme point methods EP[NG] and AEP both do slightly better than LES, especially
for the unbalanced case of $w = (1,3)$, and there is essentially no difference between EP[NG] and AEP here.

\begin{figure}[!ht]
  \centering
  \includegraphics[trim=95 30 80 35,clip,width=0.99\textwidth]{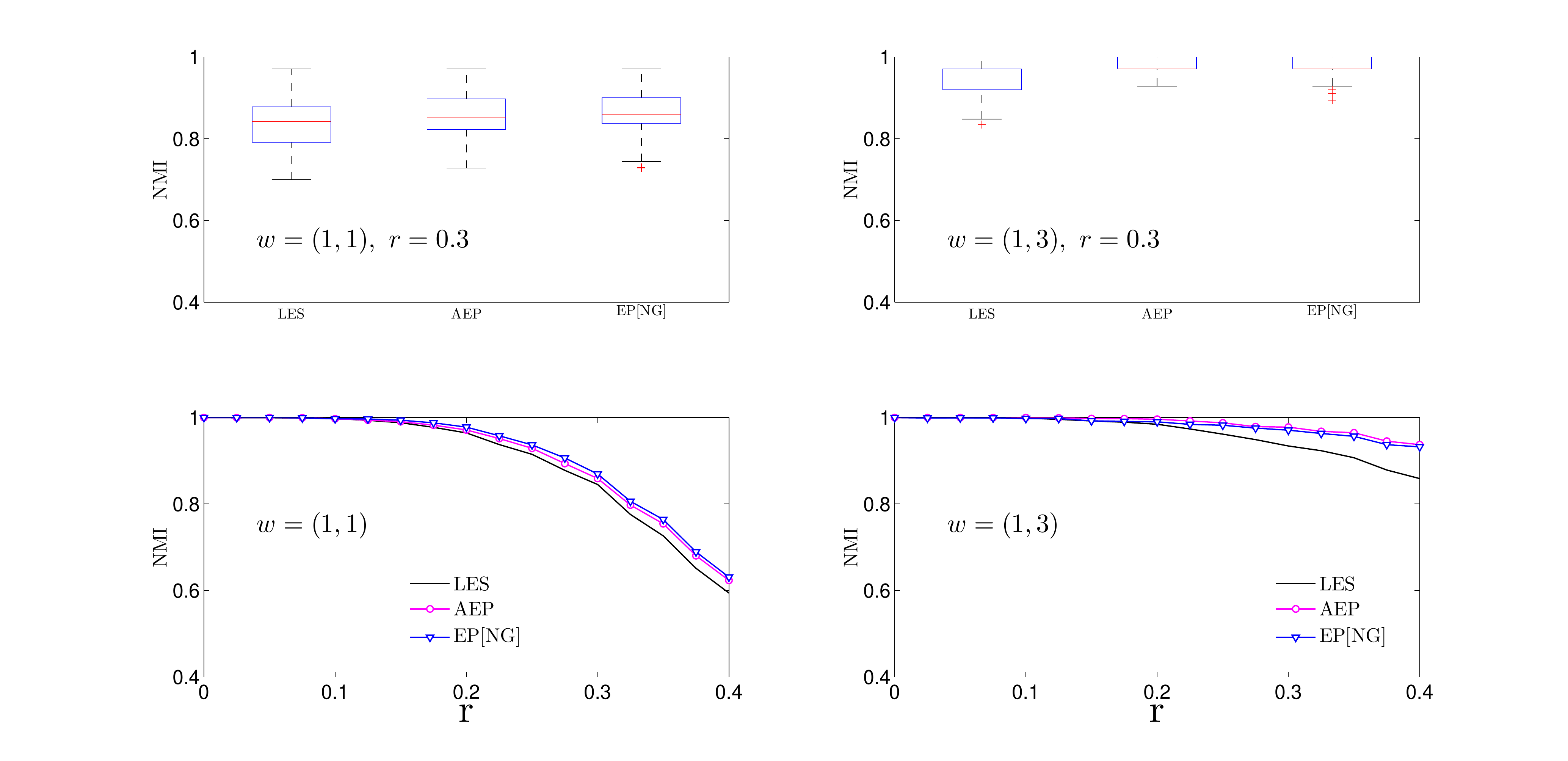}\\
  \caption{Newman-Girvan modularity.  Top row: boxplots of NMI between true and estimated
    labels.   Bottom row:  average NMI against the out-in probability
    ratio $r$. In all plots, $n_1=n_2=150$, $\lambda=15$, and $\gamma=0$.}
  \label{Fig:NewmanGirvanPlots}
\end{figure}

\subsection{Community extraction criterion}\label{SubSec:Sim:ComEXtrCrn}
Following the original extraction paper of \cite{Zhao.et.al.2011}, we generate a
community with background from the regular block model with $K = 2$, $n_1 = 60$, $n_2=240$, and the probability matrix proportional to
$$P_0 = \left(
     \begin{array}{cc}
       0.4 & 0.1 \\
       0.1 & 0.1 \\
     \end{array}
   \right) . $$
 Thus, nodes within the first community are tightly connected, while
 the rest of the nodes have equally weak links with all other nodes
 and represent the background.  We consider four values for the average expected node degree, $15$, $20$, $25$, and $30$.
Figure~\ref{Fig:ComExtrBoxplots} shows that EP[EX] performs
better than SCR and AEP, but somewhat worse than the greedy label-switching tabu search used in the original paper for maximizing the community
extraction criterion (TS).   However, the tabu search is very
computationally intensive and only feasible up to perhaps a thousand nodes, so for larger networks it is not an option at all, and no other method has been previously proposed for this problem.   The AEP method, which does not agree with AE as well as in the other cases, probably suffers from the inherent assymetry of the extraction problem.
\begin{figure}[!ht]
  \centering
  \includegraphics[trim=70 30 80 35,clip,width=0.99\textwidth]{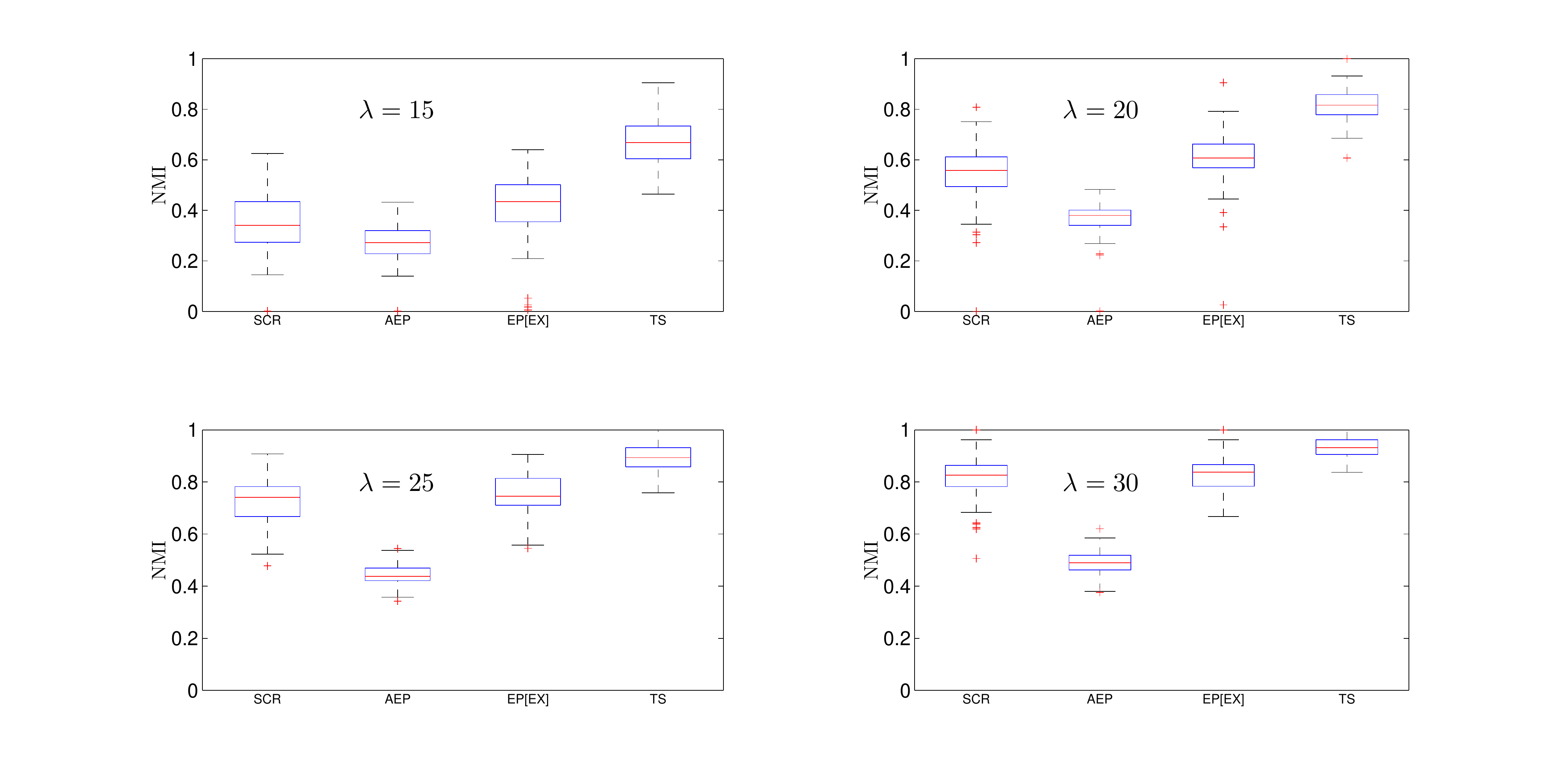}\\
  \caption{Community extraction.   The boxplots of NMI between true and estimated labels. In all plots, $n_1=60$, $n_2=240$, and $\gamma=0$.}
  \label{Fig:ComExtrBoxplots}
\end{figure}

\subsection{Real-world network data}\label{sec:data}


The first network we test our methods on, assembled by \cite{Adamic05}, consists of blogs about US politics and
hyperlinks between blogs. Each blog has been manually labeled as
either liberal or conservative, which we use as the ground truth.  Following \cite{Karrer10}, and \cite{Zhaoetal2012}, we ignore directions of the
hyperlinks and only examine the largest connected component of this
network, which has 1222 nodes and 16,714
edges, with the average degree of approximately 27.
Table ~\ref{Table:PolBlogsNet} and Figure~\ref{Fig:PolBlogs1} show the
performance of different methods.    While AEP, EP[DC], and CPL give
reasonable results, SCR, UPL,
and EP[BM] clearly miscluster the nodes.   This is consistent with
previous analyses which showed that the degree correction has to be
used for this network to achieve the correct partition, because of the presense of hub nodes.

\begin{table}[!ht]
\renewcommand{\arraystretch}{2}
\centering
\caption{The NMI between true and estimated labels for real-world networks.}
\label{Table:PolBlogsNet}
\begin{tabular}{c|c|c|c|c|c|c}
  \hline
  \textbf{Method} & SCR & AEP  & EP[BM] & EP[DC] & UPL & CPL \\ \hline
  \text{Blogs} & 0.290 &0.674 & 0.278 & 0.731 & 0.001 & 0.725 \\ \hline
  \text{Dolphins}   & 0.889 &0.814   &0.889 & 0.889 & 0.889 & 0.889 \\ \hline
\end{tabular}
\end{table}

\begin{figure}[!ht]
\centering
\subfiguretopcaptrue
\subfigure[True Labels]{
    \includegraphics[scale=0.2]{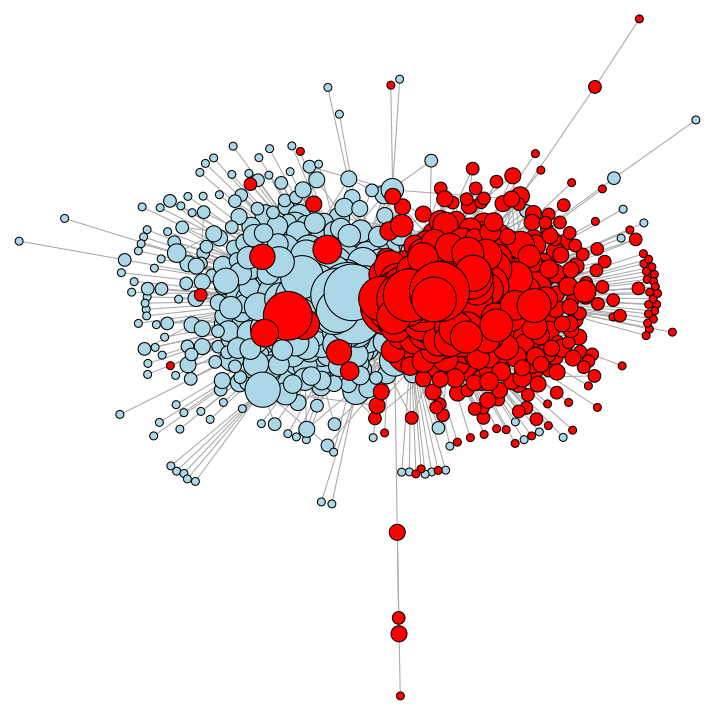}
    \label{Fig:PolBlogsTrue1}
}\\

\subfigure[UPL]{
    \includegraphics[scale=0.2]{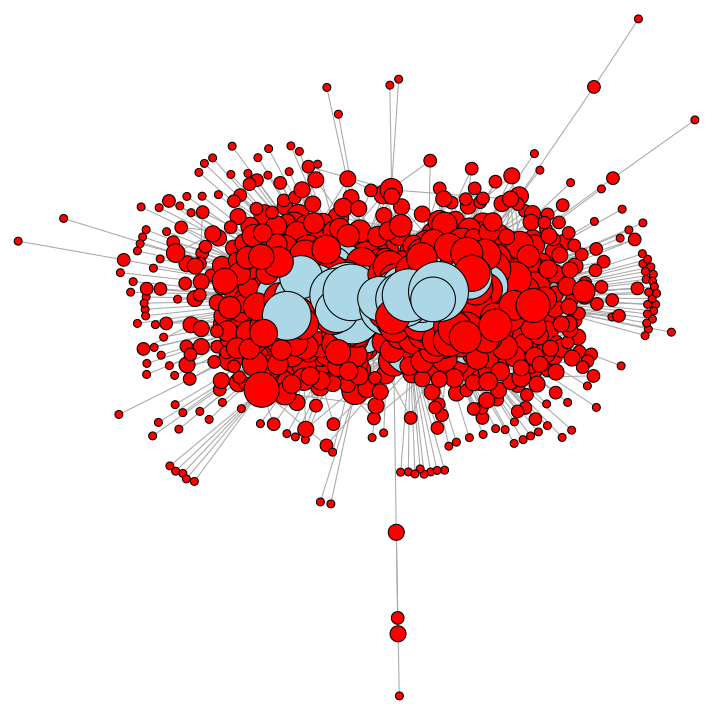}
    \label{Fig:PolBlogsUPL}
}
\subfigure[CPL]{
    \includegraphics[scale=0.2]{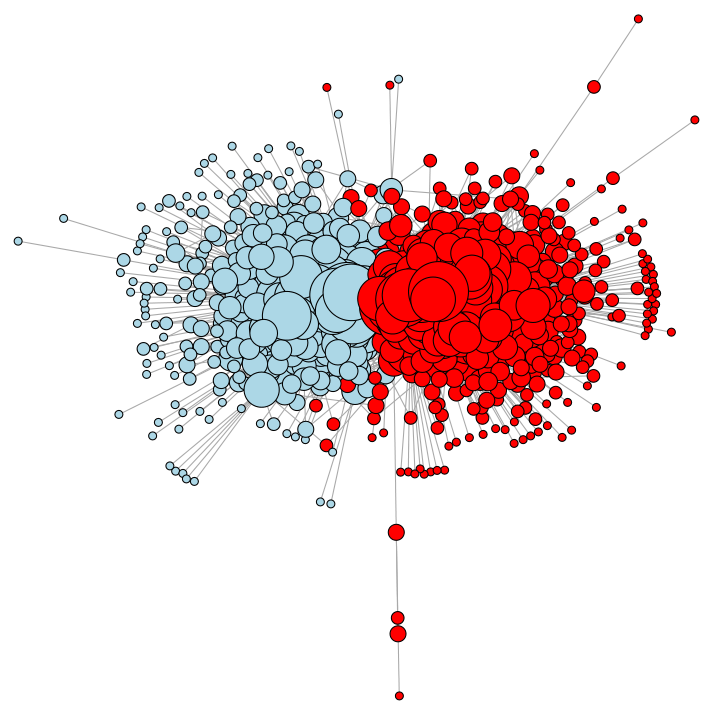}
    \label{Fig:PolBlogsCPL}
}
\subfigure[SCR]{
    \includegraphics[scale=0.2]{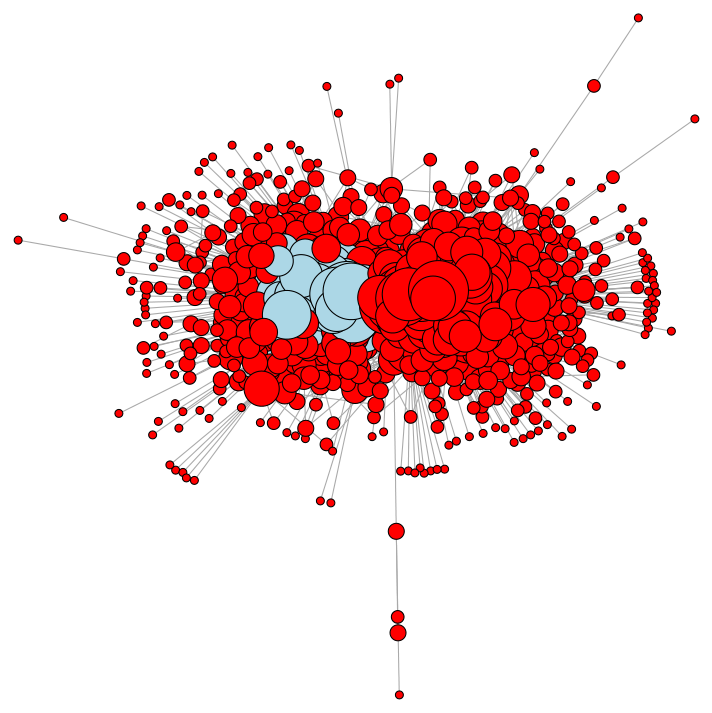}
    \label{Fig:PolBlogsSC}
}
\\
\subfigure[EP(BM)]{
    \includegraphics[scale=0.2]{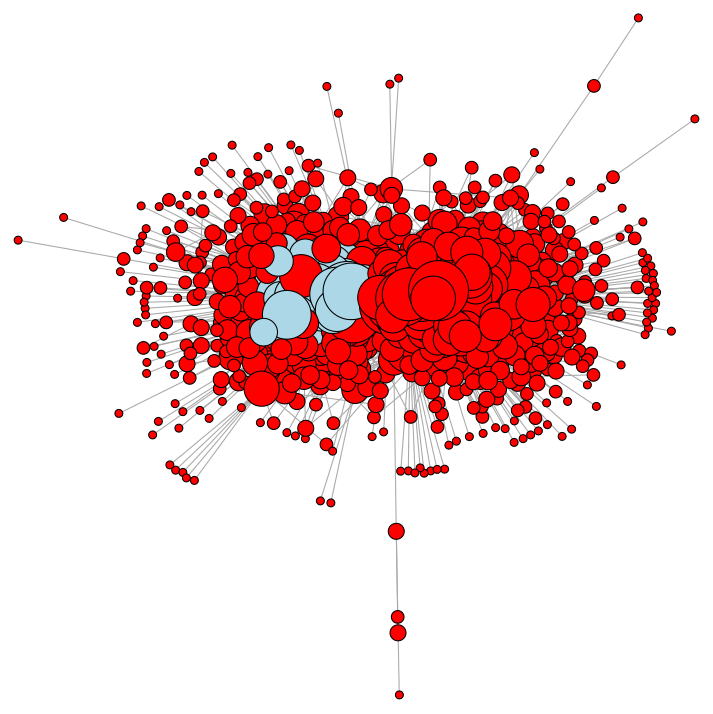}
    \label{Fig:PolBlogsEPBM}
}
\subfigure[EP(DC)]{
    \includegraphics[scale=0.2]{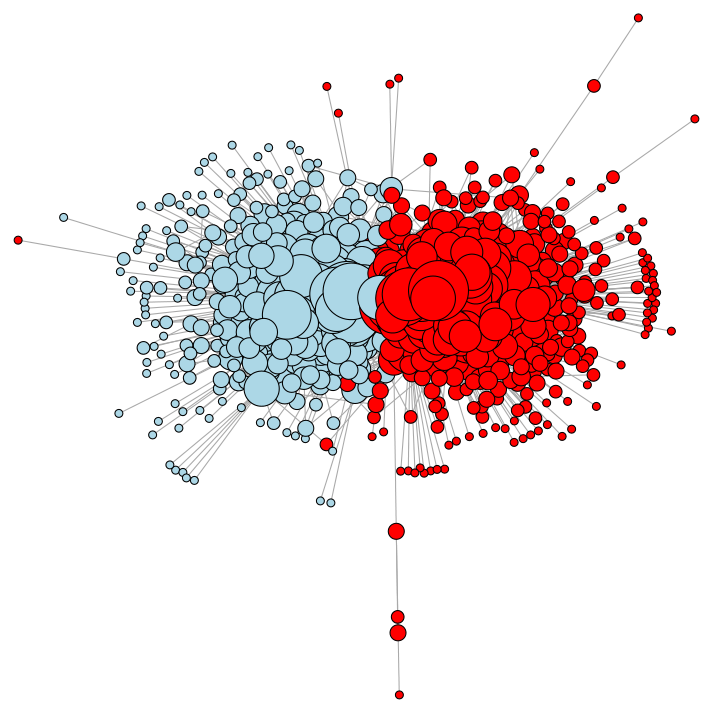}
    \label{Fig:PolBlogsEPDCBM}
}
\subfigure[AEP]{
    \includegraphics[scale=0.2]{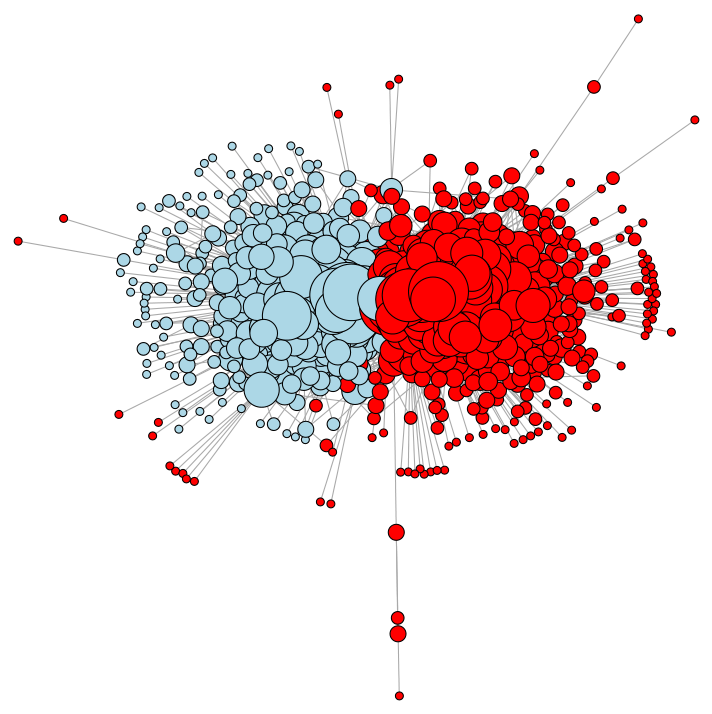}
    \label{Fig:PolBlogsAEP}
}

\caption{The network of political blogs. Node diameter is proportional to the logarithm of its degree and the colors represent community labelss.}
\label{Fig:PolBlogs1}
\end{figure}

The second network we study represents social ties between 62 bottlenose dolphins living in
Doubtful Sound, New Zealand \cite{Lusseau2003, Lusseau2004}. At some point during the study,
one well-connected dolphin (SN100) left the group, and the group split into two
separate parts, which we use as the ground truth in this example.
Table~\ref{Table:PolBlogsNet} and Figure~\ref{Fig:Dolphin} show the performance of different methods. In Figure~\ref{Fig:Dolphin},
node shapes represent the actual split, while the colors represent the estimated label.
The star-shaped node is the dolphin SN100 that left the  group. Excepting that dolphin, SCR, EP[BM], EP[DC], UPL, and CPL all miscluster one node, while AEP misclusters two nodes.   Since this small network can be well modelled by the SBM, there is no difference between DCSBM and SBM based methods, and all methods perform well.

\begin{figure}[!ht]
\centering
\subfiguretopcaptrue

\subfigure[AEP]{
    \includegraphics[scale=0.3]{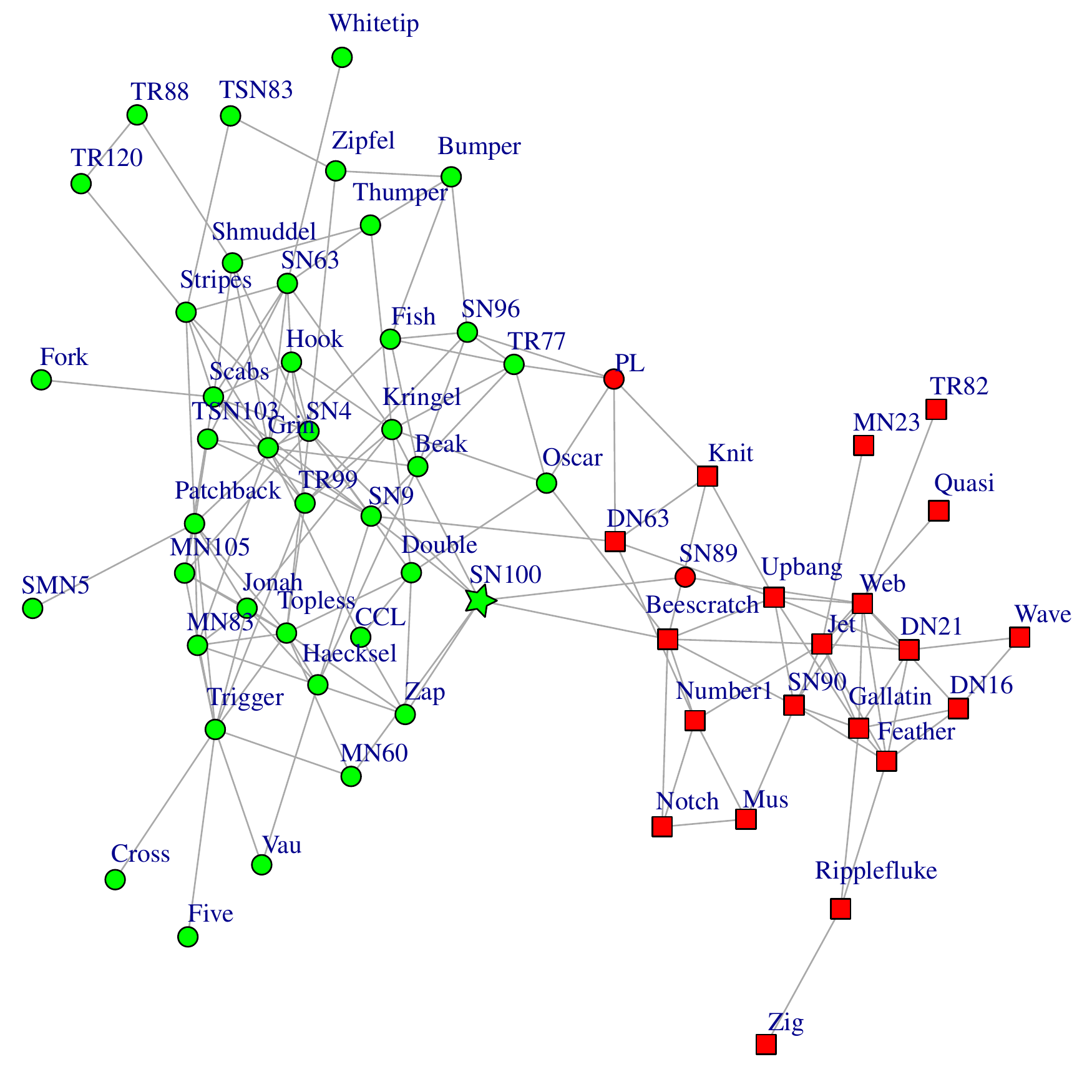}
    \label{Fig:DolphinAEP}
}
\subfigure[SCR, EP, UPL, CPL]{
    \includegraphics[scale=0.3]{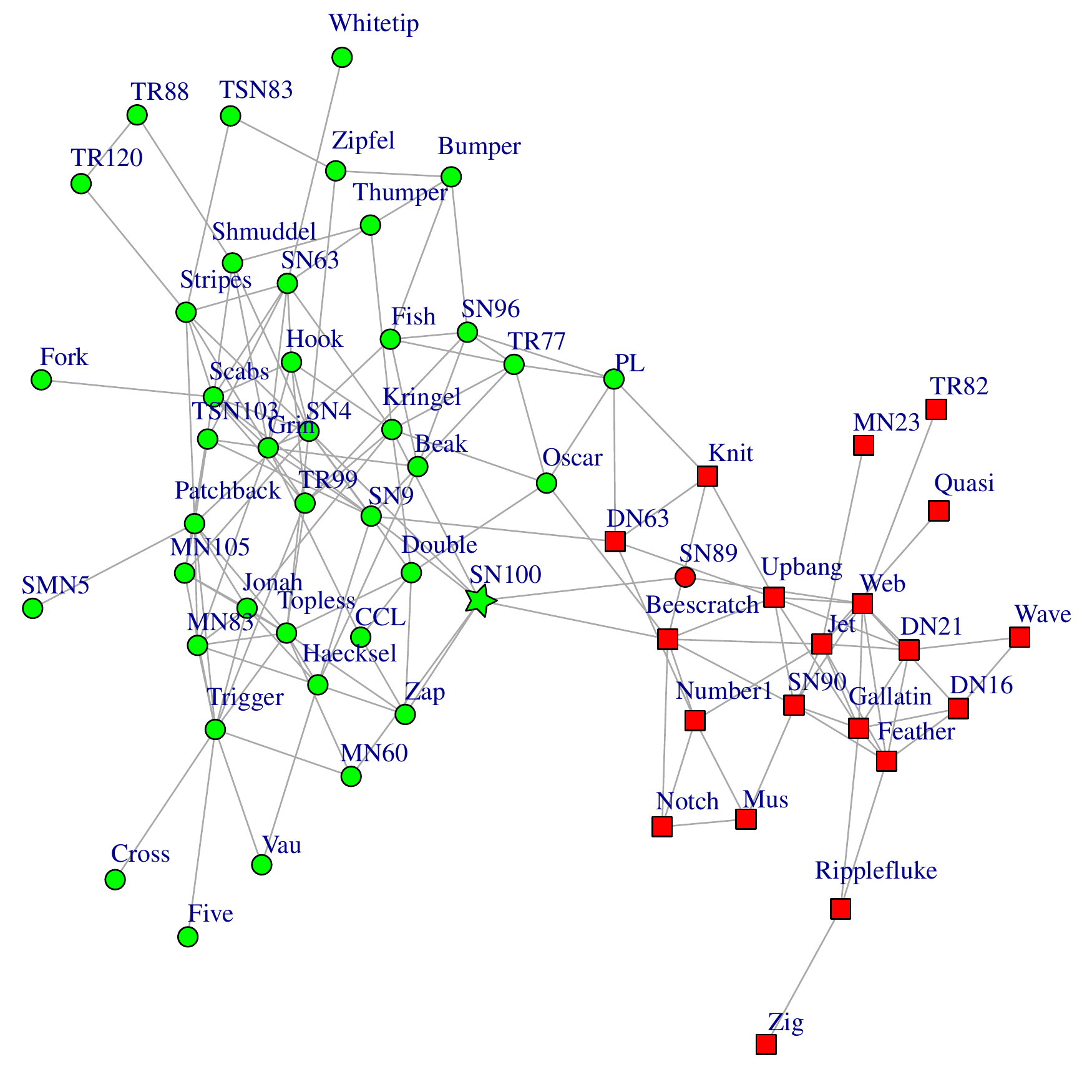}
    \label{Fig:DolphinSCR}
}
\caption{The network of 62 bottlenose dolphins.  Node shapes represent the split after the dolphin SN100 (represented by the star) left the group.  Node colors represent their estimated labels.}
\label{Fig:Dolphin}
\end{figure}

\section*{Acknowledgments}
We thank the Associate Editor and three anonymous referees for detailed and constructive feedback which led to many improvements.  We also thank Yunpeng Zhao (George Mason University) for sharing his code
for the tabu search,  and Arash A. Amini (UCLA) for sharing his code for the pseudo-likelihood
methods and helpful discussions. E.L. is partially supported by NSF grants DMS-01106772 and
DMS-1159005.  R.V. is partially supported by NSF grants DMS 1161372, 1001829, 1265782 and USAF
Grant FA9550-14-1-0009.

\bibliography{allref}
\bibliographystyle{abbrv}
\appendix

\section{Proof of results in Section 2}\label{AppendixA}

The following Lemma bounds the Lipschitz constants of $h_{B,j}$ and $f_B$ on $U_{B}[-1,1]^n$.

\begin{lemma}\label{Lem:LipConst}
Assume that Assumption ($1$) holds. For any $j\leq\kappa$ (see \ref{Eq:GenFuncType}), and $x,y\in U_{B}[-1,1]^n$, we have
\begin{eqnarray*}
  \big|h_{B,j}(x)-h_{B,j}(y)\big| &\leq& 4 \sqrt{n}\|B\|\cdot\|x-y\|, \\
  \big|f_B(x)-f_B(y)\big| &\leq& M\sqrt{n}\log(n)\|B\|\cdot\|x-y\|,
\end{eqnarray*}
where $M$ is a constant independent of $n$.
\end{lemma}
\begin{proof}[Proof of Lemma~\ref{Lem:LipConst}]
Let $e,s\in [-1,1]^n$ such that $x=U_Be,y=U_Bs$ and denote $L = \big|h_{B,j}(x)-h_{B,j}(y)\big|$. Then
\begin{eqnarray*}
   L &=& \big|(e+s_{j1})^T B(e+s_{j2}) -(s+s_{j1})^T B(s+s_{j2})\big| \\
     &=& \big| e^TB(e-s) + (e-s)^T B s +(s_{j2}+s_{j1})^T B (e-s)\big| \nonumber \\
     &\leq& 4\sqrt{n} \|B(e-s)\|.
\end{eqnarray*}
Let $B=\sum_{i=1}^m\rho_i u_i u_i^T$ be the eigendecomposition of $B$. Then
\begin{eqnarray*}
  \|B(e-s)\|^2 &=& \Big\|\sum_{i=1}^m\rho_i u_i u_i^T(e-s)\Big\|^2 = \Big\|\sum_{i=1}^m\rho_i (x_i-y_i) u_i\Big\|^2\\
  &=& \sum_{i=1}^m\rho_i^2 (x_i-y_i)^2 \leq \|B\|^2 \sum_{i=1}^m (x_i-y_i)^2 = \|B\|^2 \cdot\|x-y\|^2.
\end{eqnarray*}
Therefore $L\leq 4\sqrt{n}\|B\|\cdot\|x-y\|$. Since $h_{B,j}$ are quadratic, they are of order $O(n^2)$. Hence by Assumption ($1$), the Lipschitz constants of $g_j$ are of order $\log(n)$. Therefore
$$\big|f_B(x)-f_B(y)\big|\leq 4 \sqrt{n}\log(n)\|B\|\cdot\|x-y\|,$$
which completes the proof.
\end{proof}

In the following proofs we use $M$ to denote a positive constant independent of $n$ the value of which may change from line to line.

\begin{proof}[Proof of Lemma~\ref{Lem:MaxFuncEst}]
Since $\|e+s_{j1}\|\leq 2\sqrt{n}$ and $\|e+s_{j2}\|\leq 2\sqrt{n}$,
\begin{eqnarray*}
  |h_{A,j}(e)-h_{B,j}(e)| &=& |(e+s_{j1})^T (A-B) (e+s_{j2})| \\
   &\leq& 4n\|A-B\|.
\end{eqnarray*}
Since $h_{A,j}$ and $h_{B,j}$ are of order $O(n^2)$, $g_j^\prime$ are bounded by $\log(n)$.
Together with assumption (1) it implies that there exists $M>0$ such that
\begin{equation}\label{Ineq:fAandfEA}
  |f_A(e)-f_B(e)|\leq M n\log(n)\|A-B\|.
\end{equation}
Let $\hat{e} = \arg\max\{f_B(e),e\in\mathcal{E}_A\}$.
Then $f_A(e^*)\geq f_A(\hat{e})$ and by \eqref{Ineq:fAandfEA} we get
\begin{eqnarray}\label{Ineq:fEAhese}
  f_B(\hat{e})-f_B(e^*) &\leq&
  f_B(\hat{e}) - f_A(\hat{e}) + f_A(e^*)- f_B(e^*) \\
  \nonumber&\leq& M n\log(n)\|A-B\|.
\end{eqnarray}
Denote by $\mbox{conv}(S)$ the convex hull of a set $S$.
Then $U_Ac\in\mbox{conv}(U_A\mathcal{E}_A)$ and therefore, there exists $\eta_e\geq 0$, $\sum_{e\in\mathcal{E}_A}\eta_e=1$ such that
$$U_A c = \sum_{e\in\mathcal{E}_A}\eta_e U_A (e) =
U_A\Big(\sum_{e\in\mathcal{E}_A}\eta_e e\Big).$$
Hence
\begin{eqnarray}\label{Ineq:DisPbUcPbUclEU}
  \ \ \ \ \ \mbox{dist}\big(U_B c,\mbox{conv}(U_B\mathcal{E}_A) \big)
  &\leq& \Big\| U_B c - U_B \Big(\sum_{e\in\mathcal{E}_A}\eta_e e\Big)\Big\| \\
  \nonumber &=& \Big\| (U_B-U_A)c + (U_A-U_B) \sum_{e\in\mathcal{E}_A}\eta_e e \Big\|\\
  \nonumber &\leq& 2\sqrt{n} \ \|U_A-U_B\|.  
\end{eqnarray}
Let $y\in \mbox{conv}(U_B\mathcal{E}_A)$ be the closest point from $\mbox{conv}(U_B\mathcal{E}_A)$ to $U_Bc$, i.e.
$$\|U_Bc-y\|=\mbox{dist}\big(U_Bc,\mbox{conv}(U_B\mathcal{E}_A) \big).$$
By \ref{Ineq:DisPbUcPbUclEU} and Lemma~\ref{Lem:LipConst}, we have
\begin{equation}\label{Ineq:BeforefEAche}
  f_{B}(U_B c)-f_{B}(y)\leq M n\log(n)\|B\|\cdot\|U_A-U_B\|.
\end{equation}
The convexity of $f_B$ implies that $f_{B}(y)\leq f_B(U_B\hat{e})$, and in turn,
\begin{equation}\label{Ineq:fEAche}
  f_B(U_B c)-f_B(U_B\hat{e})\leq M n\log(n)\|B\|\cdot\|U_A-U_B\|.
\end{equation}
Note that $f_B(U_B e)=f_B (e)$ for every $e\in[-1,1]^n$. Adding \eqref{Ineq:fEAhese} and \eqref{Ineq:fEAche}, we get \eqref{Ineq:MaxFuncEst} for $T=B$. The case $T=A$ then follows from \eqref{Ineq:fAandfEA} because replacing $B$ with $A$ induces an error which is not greater than the upper bound of \eqref{Ineq:MaxFuncEst} for $T=B$.
\end{proof}

\section{Proof of Theorem 6}\label{AppendixB:FormEst}
We first present the closed form of eigenvalues and eigenvectors of $\mathbb{E}[A]$ under the regular block models.

\begin{lemma}\label{Lem:BMEigDecomp}
Under the SBM, the nonzero eigenvalues $\rho_i$ and corresponding eigenvectors $\bar{u}_i$ of $\mathbb{E}[A]$ have the following form. For $i=1,2,$
$$\rho_i=\frac{\lambda_n}{2}\left[(\pi_1+\pi_2\omega)+(-1)^{i-1} \ \sqrt{(\pi_1+\pi_2\omega)^2-4\pi_1\pi_2(\omega-r^2)}\right],$$
$$ \bar{u}_i=\frac{1}{\sqrt{n(\pi_1r_i^2+\pi_2)}}(r_i,r_i,...,r_i,1,1,...,1)^T, \ \mbox{where}$$
$$r_i=\frac{2\pi_2r}{(\pi_2\omega-\pi_1)+(-1)^i \ \sqrt{(\pi_1+\pi_2\omega)^2-4\pi_1\pi_2(\omega-r^2)}}.$$
The first $\bar{n}_1=n\pi_1$ entries of $\bar{u}_i$ equal $r_i\left(n(\pi_1r_i^2+\pi_2)\right)^{-1/2}$ and the last $\bar{n}_2=n\pi_2$ entries of $\bar{u}_i$ equal $\left(n(\pi_1r_i^2+\pi_2)\right)^{-1/2}$.
\end{lemma}
\begin{proof}[Proof of Lemma~\ref{Lem:BMEigDecomp}]
Under the SBM $\mathbb{E}[A]$ is a two-by-two block matrix with equal entries within each block.
It is easy to verify directly that $\mathbb{E}[A] \bar{u}_i =\rho_i \bar{u}_i$ for $i=1,2$.
\end{proof}

Lemma~\ref{Lem:OlivRest} bounds the difference between the eigenvalues and eigenvectors of $A$ and those of $\mathbb{E}[A]$ under the SBM.
It also provides a way to simplify the general upper bound of Theorem~\ref{Thm:GenMetdCons}.

\begin{lemma}\label{Lem:OlivRest}
Under the SBM, let $U_A$ and $U_{\mathbb{E}[A]}$ be $2\times n$ matrices whose rows are the leading eigenvectors of $A$ and $\mathbb{E}[A]$, respectively.
For any $\delta>0$, there exists a constant $M=M(r,\omega,\pi,\delta)>0$ such that if $\lambda_n>M\log(n)$ then with probability at least $1-n^{-\delta}$, we have
\begin{equation}\label{Ineq:AdjConc}
  \|A-\mathbb{E}[A]\|\leq M\sqrt{\lambda_n},
\end{equation}
\begin{equation}\label{Ineq:SubspConc}
   \|U_A-U_{\mathbb{E}[A]}\|\leq \frac{M}{\sqrt{\lambda_n}}.
\end{equation}
\end{lemma}

\begin{proof}[Proof of Lemma~\ref{Lem:OlivRest}]
Inequality \eqref{Ineq:AdjConc} follows directly from Theorem 5.2 of \cite{Lei&Rinaldo2015} and the fact that the maximum of the expected node degrees is of order $\lambda_n$.
Inequality \eqref{Ineq:SubspConc} is a consequence of \eqref{Ineq:AdjConc} and the Davis-Kahan theorem
(see Theorem VII.3.2 of \cite{Bhatia1996}) as follows.
By Lemma~\ref{Lem:BMEigDecomp}, the nonzero eigenvalues $\rho_1$ and $\rho_2$ of $\bar{A}$ are of order $\lambda_n$.
Let $$\mathcal{S}=\left[\rho_2-M\sqrt{\lambda_n},\rho_1+M\sqrt{\lambda_n}\right].$$
Then $\rho_1,\rho_2\in\mathcal{S}$ and the gap between $\mathcal{S}$ and zero is of order $\lambda_n$.
Let $\bar{P}$ be the projector onto the subspace spanned by two leading eigenvectors of $\mathbb{E}[A]$.
Since $\lambda_n$ grows faster than $\|A-\mathbb{E}[A]\|$ by \ref{Ineq:AdjConc},
only two leading eigenvalues of $A$ belong to $\mathcal{S}$.
Let $P$ be the projector onto the subspace spanned by two leading eigenvectors of $A$.
By the Davis-Kahan theorem,
$$\|U_A-U_{\mathbb{E}[A]}\| = \|\bar{P}-P\|\leq\frac{2\|A-\mathbb{E}[A]\|}{\lambda_n}
\leq \frac{2M}{\sqrt{\lambda_n}},$$
which completes the proof.
\end{proof}

Before proving Theorem~\ref{Thm:EstErrorBound} we need to establish the following lemma.

\begin{lemma}\label{Lem:EigVecErrorBound}
Let $x$, $y$, $\bar{x}$, and $\bar{y}$ be unit vectors in $\mathbb{R}^n$ such that
$\langle x,y \rangle = \langle \bar{x}, \bar{y} \rangle = 0$. Let $P$ and $\bar{P}$
be the orthogonal projections on the subspaces spanned by $\{x,y\}$ and $\{\bar{x},\bar{y}\}$ respectively.
If $\|P-\bar{P}\|\leq\epsilon$ then there exists
an orthogonal matrix $\mathcal{K}$ of size $2\times 2$ such that $||(x,y)\mathcal{K}-(\bar{x},\bar{y})||_F\leq 9\epsilon$.
\end{lemma}

\begin{proof}[Proof of Lemma~\ref{Lem:EigVecErrorBound}]
Let $x_0 = P \bar{x}$ and $y_0 = P \bar{y}$. Since $\|P-\bar{P}\|\leq\epsilon$, it follows that $\|\bar{x}-x_0\|\leq\epsilon$ and $\|\bar{y}-y_0\|\leq\epsilon$.
Let $x^\perp = \frac{x_0}{\|x_0\|}$, then
\begin{eqnarray*}
  \|\bar{x}-x^\perp\| &\leq& \|\bar{x}-x_0\| + \|x_0 - x^\perp\|
   \leq \epsilon + |1-\|x_0\|| \leq 2\epsilon.
\end{eqnarray*}
Also $\langle x^\perp, y_0 \rangle = \langle x^\perp,y_0-\bar{y}\rangle + \langle x^\perp-\bar{x},\bar{y}\rangle$
implies that $|\langle x^\perp, y_0 \rangle|\leq 3\epsilon$.
Define $z = y_0 - \langle y_0,x^\perp \rangle x^\perp$. Then $\langle z,x^\perp \rangle = 0$,
$\|\bar{y}-z\|\leq \|\bar{y}-y_0\|+\|y_0-z\|\leq 4\epsilon$, and $|1-\|z\||=|\|\bar{y}\|-\|z\||\leq 4\epsilon$. Let $y^\perp = \frac{1}{\|z\|}z$,
then
\begin{eqnarray*}
  \|\bar{y}-y^\perp\| &\leq& \|\bar{y}-z\| + \|z - y^\perp\|
   \leq 4\epsilon + |1-\|z\|| \leq 8\epsilon.
\end{eqnarray*}
Therefore $\|(\bar{x},\bar{y})-(x^\perp,y^\perp)\|_F\leq 9\epsilon$.
Finally, let $\mathcal{K}=(x,y)^T (x^\perp,y^\perp)$.
\end{proof}

\begin{proof}[Proof of Theorem~\ref{Thm:EstErrorBound}]
Denote $\varepsilon = \|U_A - U_{\mathbb{E}[A]}\|$, $U=(u_1,u_2)^T=U_A$, and $\bar{U}=(\bar{u}_1,\bar{u}_2)^T=U_{\mathbb{E}[A]}$.
We first show that there exists a constant $M>0$ such that with probability at least $1-\delta$,
\begin{equation}\label{Ineq:EstError}
  \min\Big\| (u_1^T\mathbf{1}u_2 - u_2^T\mathbf{1} u_1) \pm
  (\bar{u}_1^T\mathbf{1}\bar{u}_2 - \bar{u}_2^T\mathbf{1} \bar{u}_1)\Big\| \leq M\varepsilon\sqrt{n}.
\end{equation}
Let $\mathcal{R} =\left(\begin{smallmatrix} 0&-1\\ 1&0 \end{smallmatrix}\right)$
be the $\pi/2$-rotation on $\mathbb{R}^2$. Then
\begin{eqnarray*}
  u_1^T\mathbf{1}u_2 - u_2^T\mathbf{1} u_1 = U^T\mathcal{R}U\mathbf{1}, \ \
  \bar{u}_1^T\mathbf{1}\bar{u}_2 - \bar{u}_2^T\mathbf{1} \bar{u}_1 = \bar{U}^T\mathcal{R}\bar{U}\mathbf{1}.
\end{eqnarray*}
By Lemma~\ref{Lem:OlivRest} and Lemma~\ref{Lem:EigVecErrorBound}, there exists an orthogonal matrix $\mathcal{K}$ such that if $E=(E_1,E_2)= U^T - \bar{U}^T\mathcal{K}$ then $||E||_F\leq 9\varepsilon$.
By replacing $U^T$ with $E+\bar{U}^T\mathcal{K}$, the left hand side of \eqref{Ineq:EstError} becomes
$$\min\left\|\left(E+\bar{U}^T\mathcal{K}\right)\mathcal{R}\left(E+\bar{U}^T\mathcal{K}\right)^T\mathbf{1}\pm \bar{U}^T\mathcal{R}\bar{U}\mathbf{1}\right\|.$$
Note that $\mathcal{K}^T\mathcal{R}\mathcal{K}=\mathcal{R}$ if $\mathcal{K}$ is a rotation, and $\mathcal{K}^T\mathcal{R}\mathcal{K}=-\mathcal{R}$
if $\mathcal{K}$ is a reflection.
Therefore, it is enough to show that
$$\left\|\bar{U}^T\mathcal{K}\mathcal{R}E^T\mathbf{1} +
E\mathcal{R}\mathcal{K}^T\bar{U}\mathbf{1} + E\mathcal{R}E^T\mathbf{1}\right\|
\leq M\epsilon\sqrt{n}.$$
Note that $|E_i^T\mathbf{1}|\leq \sqrt{n}\|E_i\|\leq 9\varepsilon\sqrt{n}$ and $\|E\|_F\leq 9\varepsilon \leq 18$, so
$$\|E\mathcal{R}E^T\mathbf{1}\| = \|E_2^T \mathbf{1}E_1 - E_1^T\mathbf{1} E_2\|\leq 18^2\varepsilon\sqrt{n}.$$
From Lemma~\ref{Lem:BMEigDecomp} we see that $\bar{U}\mathbf{1}=\sqrt{n}(s_1,s_2)^T$
for some $s_1$ and $s_2$ not depending on $n$.
It follows that
$$\|E\mathcal{R}\mathcal{K}^T\bar{U}\mathbf{1}\|=\sqrt{n}\|(E_2 - E_1)\mathcal{K}^T(s_1,s_2)^T\|
\leq M\varepsilon\sqrt{n}$$
for some $M>0$.
Analogously,
$$\|\bar{U}^T\mathcal{K}\mathcal{R}E^T\mathbf{1}\| =
\|\bar{U}^T\mathcal{K} (-E_2^T\mathbf{1},E_1^T\mathbf{1})^T\| \leq M\varepsilon\sqrt{n},$$
and \eqref{Ineq:EstError} follows.
By Lemma~\ref{Lem:BMEigDecomp}, we have
$$
\bar{U}^T\mathcal{R}\bar{U}\mathbf{1}
= \alpha (\pi_2,\pi_2,...,\pi_2,-\pi_1,...,-\pi_1)^T,
$$
where $\alpha$ does not depend on $n$;
the first $n_1$ entries of $\bar{U}^T\mathcal{R}\bar{U}\mathbf{1}$ equal $\alpha\pi_2$
and the last $n_2$ entries of $\bar{U}^T\mathcal{R}\bar{U}\mathbf{1}$ equal $\alpha\pi_1$.
For simplicity, assume that in \eqref{Ineq:EstError} the minimum is when the sign is negative (because $\hat{c}$ is unique up to a factor of $-1$). If node $i$ is mis-clustered by $\hat{c}$ then
$$|(U^T\mathcal{R}U\mathbf{1})_i-(\bar{U}^T\mathcal{R}\bar{U}\mathbf{1})_i|
\geq\min_i|(\bar{U}^T\mathcal{R}\bar{U}\mathbf{1})_i|=:\eta.$$
Let $k$ be the number of mis-clustered nodes, then by \eqref{Ineq:EstError}, $\eta\sqrt{k}\leq M\varepsilon\sqrt{n}$.
Therefore the fraction of mis-clustered nodes, $k/n$, is of order $\varepsilon^2$.
If $U_A$ is formed by the leading eigenvectors of $A$, then it remains to use inequality \eqref{Ineq:SubspConc} of Lemma~\ref{Lem:OlivRest}.
\end{proof}

\section{Proof of results in Section 3}\label{supplement}
Let us first describe the projection of the cube under regular block models, which will be used to replace Assumption ($2$).
See Figure~\ref{Fig:CubeProj} for an illustration.

\begin{lemma}\label{Lem:BMCubeProj}
Consider the regular block models and let $\mathcal{R}= U_{\mathbb{E}[A]}[-1,1]^n$. Then $\mathcal{R}$ is a parallelogram;
the vertices of $\mathcal{R}$ are $\{\pm U_{\mathbb{E}[A]}(c),\pm U_{\mathbb{E}[A]}(\mathbf{1})\}$, where $c$ is a true label vector. The angle between two adjacent sides of $\mathcal{R}$ does not depend on $n$.
\end{lemma}
\begin{proof}[Proof of Lemma~\ref{Lem:BMCubeProj}]
Eigenvectors of $\mathbb{E}[A]$ are computed in Lemma~\ref{Lem:BMEigDecomp}.
Let
$$x=\left(r_1\left(n(\pi_1r_1^2+\pi_2)\right)^{-1/2},r_2\left(n(\pi_1r_2^2+\pi_2)\right)^{-1/2}\right)^T,$$
$$y= \left(\left(n(\pi_1r_1^2+\pi_2)\right)^{-1/2},\left(n(\pi_1r_2^2+\pi_2)\right)^{-1/2}\right)^T.$$
Then $\mathcal{R}=\left\{(\epsilon_1+\cdot\cdot\cdot+\epsilon_{\bar{n}_1}) x +
      (\epsilon_{\bar{n}_1+1}+\cdot\cdot\cdot+\epsilon_n) y,\epsilon_i\in[-1,1]\right\}$,
and it is easy to see that $\mathcal{R}$ is a parallelogram.
Vertices of $\mathcal{R}$ correspond to the cases when $\epsilon_1=\cdot\cdot\cdot=\epsilon_{\bar{n}_1}=\pm 1$ and $\epsilon_{\bar{n}_1+1}=\cdot\cdot\cdot=\epsilon_{n}=\pm 1$.
The angle between two adjacent sides of $\mathcal{R}$ equals the angle between $\sqrt{n} x$ and $\sqrt{n} y$, which does not depend on $n$.
\end{proof}

\subsection{Proof of results in Section \ref{SubSec:MaxLogLikDCBM}}\label{AppendixB:DCBM}
Under degree-corrected block models, let us denote by $\bar{A}$ the conditional expectation of $A$ given the degree parameters $\theta=(\theta_1,...,\theta_n)^T$.
Note that if $\theta_i\equiv 1$ then $\bar{A} = \mathbb{E}A$.
Since $\bar{A}$ depends on $\theta$, its eigenvalues and eigenvectors  may not have a closed form.
Nevertheless, we can approximate them using $\rho_i$ and $\bar{u}_i$ from Lemma~\ref{Lem:BMEigDecomp}.
To do so, we need the following lemma.

\begin{lemma}\label{Lem:EigDec}
Let $M=\rho_1 x_1 x_1^T + \rho_2 x_2 x_2^T$, where $x_1,x_2\in\mathbb{R}^n$, $\|x_1\|=\|x_2\|=1$, $\rho_1\neq 0$, and $\rho_2\neq 0$. If $c=\langle x_1,x_2\rangle$ then the eigenvalues $z_i$ and corresponding eigenvectors $y_i$ of $M$ have the following form. For $i=1,2$,
\begin{eqnarray*}
  z_i &=& \frac{1}{2}\left[(\rho_1+\rho_2)+(-1)^{i-1}\sqrt{(\rho_2-\rho_1)^2+4\rho_1\rho_2 c^2}\right], \\
  y_i &=& (c\rho_1)x_1+(z_i-\rho_1)x_2.
\end{eqnarray*}
If $\rho_1$ and $\rho_2$ are fixed, $\rho_1\geq\rho_2$, and $c=o(1)$ as $n\rightarrow\infty$ then eigenvalues and eigenvectors of $M$ have the form
\begin{eqnarray*}
  z_1 &=& \rho_1 +O(c^2), \ \ z_2 = \rho_2 + O(c^2), \\
  y_1 &=& x_1+ O(c) x_2,  \ \ y_2 = x_2 + O(c)x_1.
\end{eqnarray*}
\end{lemma}
\begin{proof}[Proof of Lemma~\ref{Lem:EigDec}]
It is easy to verify that $My_i=z_iy_i$ for $i=1,2$. The asymptotic formulas of $z_i$ and $y_i$ then follow directly from the forms of $z_i$ and $y_i$.
\end{proof}

The next lemma shows the approximation of eigenvalues and eigenvectors of $\bar{A}$.

\begin{lemma}\label{Lem:DCBMEigDecomp}
Consider the degree-corrected block models (described in Section \ref{SubSec:MaxLogLikDCBM}) and let $D_\theta = \mathrm{diag}(\theta)$.
Denote by $\bar{A}$ the conditional expectation of $A$ given $\theta$.
Then for any $\delta\in(0,1)$, with probability at least $1-\delta$, the nonzero eigenvalues $\rho_i^\theta$ and corresponding eigenvectors $\bar{u}_i^\theta$ of $\bar{A}$ have the following form.
For $i=1,2$,
$$\rho_i^\theta=\rho_i\|D_\theta \bar{u}_i\|^2 \left(1+O(1/n)\right), $$
$$ \bar{u}_1^\theta = \frac{\tilde{u}_1^\theta}{\|\tilde{u}_1^\theta\|}, \ \ \mbox{where } \
\tilde{u}_1^\theta = \frac{D_\theta \bar{u}_1}{\|D_\theta \bar{u}_1\|} +
O\left(n^{-1/2}\right)\frac{D_\theta \bar{u}_2}{\|D_\theta \bar{u}_2\|},$$
$$\bar{u}_2^\theta = \frac{\tilde{u}_2^\theta}{\|\tilde{u}_2^\theta\|}, \ \ \mbox{where } \
\tilde{u}_2^\theta = \frac{D_\theta \bar{u}_2}{\|D_\theta \bar{u}_2\|} +
O\left(n^{-1/2}\right)\frac{D_\theta \bar{u}_1}{\|D_\theta \bar{u}_1\|},$$
where $\rho_i$, $\bar{u}_i$, and $r_i$ are defined in Lemma~\ref{Lem:BMEigDecomp}.
\end{lemma}
\begin{proof}[Proof of Lemma~\ref{Lem:DCBMEigDecomp}]
Let $M=\rho_1 \bar{u}_1 \bar{u}_1^T + \rho_2 \bar{u}_2 \bar{u}_2^T$ be the expectation of the adjacency matrix in the regular block model setting.
In the degree-corrected block model setting, given $\theta$, we have
\begin{eqnarray*}
  \mathbb{E}[A] &=& D_\theta M D_\theta = \rho_1 D_\theta \bar{u}_1(D_\theta \bar{u}_1)^T+
\rho_2D_\theta \bar{u}_2(D_\theta \bar{u}_2)^T \\
 &=& \rho_1\|D_\theta \bar{u}_1\|^2\frac{D_\theta \bar{u}_1}{\|D_\theta \bar{u}_1\|}
  \frac{(D_\theta \bar{u}_1)^T}{\|D_\theta \bar{u}_1\|} +
  \rho_2\|D_\theta \bar{u}_2\|^2\frac{D_\theta \bar{u}_2}{\|D_\theta \bar{u}_2\|}
  \frac{(D_\theta \bar{u}_2)^T}{\|D_\theta \bar{u}_2\|}.
\end{eqnarray*}
We are now in the setting of Lemma~\ref{Lem:EigDec} with
\begin{eqnarray*}
  c &=& \big(\|D_\theta \bar{u}_1\|\|D_\theta \bar{u}_2\|\big)^{-1}\langle D_\theta \bar{u}_1, D_\theta \bar{u}_2\rangle \\
   &=& c_\theta
  \left[\pi_1\sqrt{(\pi_1r_1^2+\pi_2)(\pi_1r_2^2+\pi_2)}\|D_\theta \bar{u}_1\|\|D_\theta \bar{u}_2\|\right]^{-1}, \\
  \mbox{where \ }c_\theta &=& \frac{1}{n}\left[\pi_1(\theta_{\bar{n}_1+1}^2+\cdot\cdot\cdot+\theta_n^2)-
\pi_2(\theta_1^2+\cdot\cdot\cdot+\theta_{\bar{n}_1}^2)\right].
\end{eqnarray*}
Note that the two sums in the formula of $c_\theta$ have the same expectation.
It remains to apply Hoeffding's inequality to each sum.
\end{proof}

Since we do not have closed-form formulas for eigenvectors of $\bar{A}$, we can not describe $U_{\bar{A}}[-1,1]^n$ explicitly. Lemma~\ref{Lem:DCBMCubeProj} provides an approximation of $U_{\bar{A}}[-1,1]^n$. It will be used to replace Assumption ($2$).

\begin{lemma}\label{Lem:DCBMCubeProj}
Consider the setting of Lemma~\ref{Lem:DCBMEigDecomp} and let $\mathcal{R}^\theta = U_{\bar{A}}[-1,1]^n$ and
\begin{equation}\label{Eq:ApprCubProj}
  \hat{\mathcal{R}}^\theta = \mathrm{conv}\left\{\pm U_{\bar{A}}(c),\pm U_{\bar{A}}(\mathbf{1})\right\}.
\end{equation}
Then $\hat{\mathcal{R}}^\theta$ is a parallelogram and the angle between two adjacent sides is bounded away from zero and $\pi$;
$\mathcal{R}^\theta$ is well approximated by $\hat{\mathcal{R}}^\theta$ in the sense that
$$\mathrm{dist}\left(\mathcal{R}^\theta,\hat{\mathcal{R}}^\theta\right)
= \sup_{x\in\mathcal{R}^\theta}\inf_{y\in \hat{\mathcal{R}}^\theta} \|x-y\| =O(1)$$
as $n\rightarrow\infty$.
\end{lemma}

\begin{proof}[Proof of Lemma~\ref{Lem:DCBMCubeProj}]
Let $v_i = \|D_\theta \bar{u}_i\|^{-1}D_\theta \bar{u}_i$, $i=1,2$, $V=(v_1,v_2)^T$, and $\mathcal{R}_V = V[-1,1]^n$.
Following the same argument in the proof of Lemma~\ref{Lem:BMCubeProj}, it is easy to show that $\mathcal{R}_V$ is a parallelogram with vertices $\left\{\pm V c,\pm V\mathbf{1}\right\}$.
By Lemma~\ref{Lem:DCBMEigDecomp}, $\|v_i-\bar{u}_i^\theta\|=O(n^{-1/2})$, which in turn implies
$ \ \mathrm{dist}\left(\mathcal{R}^\theta,\mathcal{R}_V\right)=O(1)$.
The distance between two parallelograms $\mathcal{R}_V$ and $\hat{\mathcal{R}}^\theta$ is bounded by the maximum of the distances between corresponding vertices, which is also of order $O(1)$ because $\|v_i-\bar{u}_i^\theta\|=O(n^{-1/2})$.
Finally by triangle inequality
$$\mathrm{dist}\left(\hat{\mathcal{R}}^\theta,\mathcal{R}^\theta\right)
\leq \mathrm{dist}\left(\hat{\mathcal{R}}^\theta,\mathcal{R}_V\right)+
\mathrm{dist}\left({\mathcal{R}}_V,\mathcal{R}^\theta\right)=O(1).$$
The angle between two adjacent sides of $\mathcal{R}_V$ equals the angle between $\sqrt{n} x$ and $\sqrt{n} y$, where $x$ and $y$ are defined in the proof of Lemma~\ref{Lem:BMCubeProj}, which does not depend on $n$. Since $\mathrm{dist}(\hat{\mathcal{R}}^\theta,\mathcal{R}_V)=O(1)$, the angle between two adjacent sides of $\hat{\mathcal{R}}^\theta$ is bounded from zero and $\pi$.
\end{proof}

Before showing properties of the profile log-likelihood, let us introduce some new notations.
Let $\bar{O}_{11}$, $\bar{O}_{12}$, $\bar{O}_{22}$, and $\bar{Q}_{DC}$ be the population version of $O_{11}$, $O_{12}$, $O_{22}$, and $Q_{DC}$, when $A$ is replaced with $\bar{A}$. We also use $\bar{Q}_{BM}$, $\bar{Q}_{NG}$, and $\bar{Q}_{EX}$ to denote the population version of $Q_{BM}$, $Q_{NG}$, and $Q_{EX}$ respectively. The following discussion is about $\bar{Q}_{DC}$, but it can be carried out for $\bar{Q}_{BM}$, $\bar{Q}_{NG}$, and $\bar{Q}_{EX}$ with obvious modifications and the help of Lemma~\ref{Lem:Formn1n2}.

Note that $\bar{O}_{11}$, $\bar{O}_{12}$, and $\bar{O}_{22}$ are quadratic forms of $e$ and $\bar{A}$, therefore $\bar{Q}_{DC}$ depends on $e$ through $U_{\bar{A}}e$, where $U_{\bar{A}}$ is the $2\times n$ matrix whose rows are eigenvectors of $\bar{A}$. With a little abuse of notation, we also use $\bar{O}_{ij}$, $i,j=1,2$,  and $\bar{Q}_{DC}$ to denote the induced functions on $U_{\bar{A}}[-1,1]^n$. Thus, for example if $x\in U_{\bar{A}}[-1,1]^n$ then  $\bar{Q}_{DC}(x)=\bar{Q}_{DC}(U_{\bar{A}}e)$ for any $e\in[-1,1]^n$ such that $x=U_{\bar{A}}e$.

To simplify $\bar{Q}_{DC}$, let $\rho_1^\theta$ and $\rho_2^\theta$ be eigenvalues of $\bar{A}$ as in Lemma~\ref{Lem:DCBMEigDecomp} and let
\begin{eqnarray*}
    t = (t_1,t_2)^T = U_{\bar{A}} \mathbf{1}, \ \mu = (\rho_1^\theta t_1, \rho_2^\theta t_2)^T.
\end{eqnarray*}
We parameterize $x\in U_{\bar{A}}[-1,1]^n$ by $x  = \alpha t + \beta v$, where $v=(v_1,v_2)^T$ is a unit vector perpendicular to $\mu$.
If we denote $a=\frac{1}{4}(\rho_1^\theta t_1^2+\rho_2^\theta t_2^2)$ and $b=\frac{1}{4}(\rho_1^\theta v_1^2+\rho_2^\theta v_2^2)$, then
\begin{eqnarray*}
\bar{O}_{11} &=& (\alpha+1)^2 a + \beta^2 b,\
\bar{O}_{22} = (\alpha-1)^2 a + \beta^2 b,\
\bar{O}_{12} = (1-\alpha^2) a - \beta^2 b, \\
\bar{O}_1 &=& \bar{O}_{11}+ \bar{O}_{12} = 2(1+\alpha)a,\
\bar{O}_2 = \bar{O}_{22}+ \bar{O}_{12} = 2(1-\alpha)a.
\end{eqnarray*}
Note that $\bar{O}_{11}\bar{O}_{22} - \bar{O}_{12}^2 = 4\beta^2 ab>0$ since $\rho_1^\theta$ and $\rho_2^\theta$ are positive by Lemma~\ref{Lem:DCBMEigDecomp}. With a little abuse of notation, we also use $\bar{Q}_{DC}(\alpha,\beta)$ to denote the value of $\bar{Q}_{DC}$ in the $(\alpha,\beta)$ coordinates described above. We now show some properties of $\bar{Q}_{DC}$.

\begin{lemma}\label{Lem:DCBMLogLikProp}
Consider $\bar{Q}=\bar{Q}_{DC}$ on $\hat{\mathcal{R}}^\theta$ defined by \eqref{Eq:ApprCubProj}. Then
\begin{description}
  \item[($a$)] $\bar{Q}(\alpha,0)$ is a constant.
  \item[($b$)] $\frac{\partial^2 \bar{Q}} {\partial\beta^2}\geq 0$, $\frac{\partial \bar{Q}} {\partial\beta}>0$ if $\beta>0$ and $\frac{\partial \bar{Q}} {\partial\beta}<0$ if $\beta<0$.
      Thus, $\bar{Q}$ achieves minimum when $\beta=0$ and maximum on the boundary of $\hat{\mathcal{R}}^\theta$.
  \item[($c$)] $\bar{Q}$ is convex on the boundary of $\hat{\mathcal{R}}^\theta$. Thus, $\bar{Q}$ achieves maximum at $\pm U_{\bar{A}}(c)$.
  \item[($d$)] For any $x\in U_{\bar{A}}[-1,1]^n$, if $\bar{Q}(U_{\bar{A}}(c))-\bar{Q}(x)\leq \epsilon$ then
      $$\|U_{\bar{A}}(c)-x\| \leq 4\epsilon\sqrt{n}\left(\bar{Q}(U_{\bar{A}}(c))
      -\min_{\hat{\mathcal{R}}^\theta} \bar{Q}\right)^{-1}.$$
  \item[($e$)] For any $\delta\in(0,1)$, $\max_{\hat{\mathcal{R}}^\theta} \bar{Q}-\min_{\hat{\mathcal{R}}^\theta} \bar{Q}$ is of order $n\lambda_n$ with probability at leat $1-\delta$.

\end{description}
\end{lemma}
Parts (a) and (b) are used to prove part (c), which together with Lemma~\ref{Lem:DCBMCubeProj} will be used to replace Assumption (2). Parts (d) verifies Assumption (4), and part (e) provides a way to simplify the upper bound in part (d).

\begin{proof}[Proof of Lemma~\ref{Lem:DCBMLogLikProp}]
Note that because $\hat{\mathcal{R}}^\theta\subset\mathcal{R}^\theta$, $\bar{O}_{11}$, $\bar{O}_{12}$, and $\bar{O}_{22}$ are nonnegative on $\hat{\mathcal{R}}^\theta$.
Also, if we multiply $\bar{O}_{11}$, $\bar{O}_{12}$, and $\bar{O}_{22}$ by a constant $\eta>0$ then the resulting function has the form $\eta\bar{Q} + C$, where $C$ is a constant not depending on $(\alpha,\beta)$, and therefore the behavior of $\bar{Q}$ that we are interested in does not change. In this proof we use $\eta=1/a$.
Since $\bar{Q}$ is symmetric with respect to $\beta$, after multiplying by $1/a$, we replace $\beta^2 b/a$ with $\beta$ and only consider $\beta\geq 0$. Thus, we may assume that
\begin{eqnarray}\label{Eq:OSimpd}
\ \ \bar{O}_{11} &=& (\alpha+1)^2 + \beta,\
\bar{O}_{22} = (\alpha-1)^2 + \beta,\
\bar{O}_{12} = (1-\alpha^2) - \beta, \\
\bar{O}_1 &=& \bar{O}_{11}+ \bar{O}_{12} = 2(1+\alpha),\
\bar{O}_2 = \bar{O}_{22}+ \bar{O}_{12} = 2(1-\alpha).\nonumber
\end{eqnarray}

($a$) With \eqref{Eq:OSimpd} and $\beta=0$, it is straightforward to verify that $Q(\alpha,0)$ does not depend on $\alpha$.

($b$) Simple calculation shows that
$$\frac{\partial \bar{Q}}{\partial\beta}=\log\frac{\bar{O}_{11}\bar{O}_{22}}{\bar{O}_{12}^2}\geq 0, \ \
\frac{\partial^2 \bar{Q}}{\partial\beta^2}=\frac{1}{\bar{O}_{11}}+\frac{1}{\bar{O}_{22}}+\frac{2}{\bar{O}_{12}}\geq 0.$$

($c$) We show that $\bar{Q}$ is convex on the boundary line connecting $U_{\bar{A}}(\mathbf{1})$ and $U_{\bar{A}}(c)$.
Let $(\alpha_0,\beta_0)^T$ be the coordinates of $U_{\bar{A}}(c)$, where $\beta_0>0$ and $\alpha_0\in(-1,1)$.
We parameterize the segment connecting $U_{\bar{A}}(c)$ and $U_{\bar{A}}(\mathbf{1})$ by
\begin{eqnarray}\label{SegPar}
  \left\{\left(\alpha,\frac{\beta_0(1-\alpha)}{1-\alpha_0}\right)^T, \ \alpha\in [\alpha_0,1]\right\}.
\end{eqnarray}
With this parametrization, $\bar{O}_{11}$, $\bar{O}_{12}$, and $\bar{O}_{22}$ have the forms
\begin{eqnarray*}
  \bar{O}_{11} &=& (\alpha +1)^2+\rho(\alpha-1)^2, \ \bar{O}_{22} = (\alpha -1)^2+\rho(\alpha-1)^2 \\
  \bar{O}_{12} &=& (1-\alpha^2)-\rho(\alpha-1)^2, \ \rho = \frac{\beta_0^2}{(1-\alpha_0)^2}.
\end{eqnarray*}
Simple calculation shows that
\begin{eqnarray*}
  \frac{1}{2}\frac{d^2\bar{Q}}{d\alpha^2} &=& (\rho+1)\log\frac{(\rho+1)\bar{O}_{11}}{[\alpha+1+\rho(\alpha-1)]^2} \\
  &+& \frac{4\rho}{[\alpha+1][\alpha+1+\rho(\alpha-1)]}-\frac{8\rho}{\bar{O}_{11}}.
\end{eqnarray*}
Note that the value of the right-hand side at $\alpha=1$ is $(\rho+1)\log(\rho+1)-\rho\geq0$ for any $\rho\geq 0$. Therefore to show that $\frac{d^2\bar{Q}}{d\alpha^2}\geq 0$, it is enough to show that $\frac{d^2\bar{Q}}{d\alpha^2}$ is non-increasing.
Simple calculation shows that
\begin{eqnarray*}
  \frac{d^3\bar{Q}}{d\alpha^3} &=& 16\rho^2\left[(\alpha-1)^2\rho+\alpha^2-2\alpha-3\right]\times\\
   &\times&\left[(3\alpha+1)(\alpha-1)\rho+3(\alpha+1)^2\right]D^{-1},
\end{eqnarray*}
where $D=\bar{O}_{11}^2(\alpha+1)^2\left[\alpha+1+\rho(\alpha-1)\right]^2$.
Since $\rho(1-\alpha)\leq (1+\alpha)$ because $\bar{O}_{12}\geq 0$, it follows that
$$(\alpha-1)^2\rho+\alpha^2-2\alpha-3\leq (1-\alpha)(1+\alpha)+\alpha^2-2\alpha-3=-2(\alpha+1)\leq 0.$$
Note that if $(3\alpha+1)(\alpha-1)\geq 0$ then $(3\alpha+1)(\alpha-1)\rho+3(\alpha+1)^2\geq 0$.
Otherwise $3\alpha+1\geq 0$ and since $\rho(\alpha-1)\geq -(1+\alpha)$, it follows that
$$(3\alpha+1)(\alpha-1)\rho+3(\alpha+1)^2\geq -(3\alpha+1)(\alpha+1)+3(\alpha+1)^2=2(\alpha+1)\geq 0.$$
Thus $\frac{d^3\bar{Q}}{d\alpha^3}\leq 0$.
We have shown that $\bar{Q}$ is convex on the segment connecting $U_{\bar{A}}(c)$ and $U_{\bar{A}}(\mathbf{1})$.
The same argument applies for other sides of the boundary of $\hat{\mathcal{R}}^\theta$.

($d$) Let $(\alpha_x,\beta_x)$ be the parameters of $x$, $\hat{x}$ be the point with parameters $(\alpha_x,0)$, and $x^*$ be the point on the boundary of $\hat{\mathcal{R}}_U$ with parameters $(\alpha_x,\beta_x^*)$.
Without loss of generality we assume that $x^*$ is on the line connecting $x_c=U_{\bar{A}}(c)$ and $x_{\mathbf{1}}=U_{\bar{A}}(\mathbf{1})$.
Note that ($a$),($b$), and ($c$) imply
$$\bar{Q}(x_c)\geq \bar{Q}(x^*)\geq \bar{Q}(x)\geq \bar{Q}(\hat{x})=\bar{Q}(x_{\mathbf{1}}).$$
Let $\ell = \bar{Q}(x_c)-\min_{\hat{\mathcal{R}}^\theta} \bar{Q}$. Since $\bar{Q}(\alpha_x,\beta)$ is convex in $\beta$ (by ($b$)), we have
\begin{eqnarray*}
  \frac{\|x^*-x\|}{\|x^*-\hat{x}\|} \leq \frac{\bar{Q}(x^*)-\bar{Q}(x)}{\bar{Q}(x^*)-\bar{Q}(\hat{x})}
  \leq \frac{\bar{Q}(x_c)-\bar{Q}(x)}{\bar{Q}(x_c)-\bar{Q}(\hat{x})}\leq \frac{\epsilon}{\ell}.\\
\end{eqnarray*}
Therefore $\|x^*-x\|\leq \epsilon\ell^{-1}\|x^*-\hat{x}\|\leq 2\epsilon\sqrt{n}\ell^{-1}$.
Since $\bar{Q}$ is convex on the boundary of $\hat{\mathcal{R}}^\theta$, we have
$$\frac{\|x_c-x^*\|}{\|x_c-x_{\mathbf{1}}\|}\leq\frac{\bar{Q}(x_c)-\bar{Q}(x^*)}{\bar{Q}(x_c)-\bar{Q}(x_{\mathbf{1}})}
\leq\frac{\bar{Q}(x_c)-\bar{Q}(x)}{\bar{Q}(x_c)-\bar{Q}(x_{\mathbf{1}})}\leq\frac{\epsilon}{\ell},$$
which in turn implies $\|x_c-x^*\|\leq\epsilon\ell^{-1}\|x_c-x_{\mathbf{1}}\|\leq 2\epsilon\sqrt{n}\ell^{-1}$.
Finally by triangle inequality
$$\|x_c-x\|\leq \|x_c-x^*\|+\|x^*-x\|\leq 4\epsilon\sqrt{n}\ell^{-1}.$$

($e$) Note that $\min_{\hat{\mathcal{R}}^\theta} \bar{Q} = \bar{Q}(\alpha_0,0)=\bar{Q}(0,0)$.
Also, to find $\bar{Q}(c)-\bar{Q}(0)$ we do not have to calculate $\bar{O}_1\log \bar{O}_1+\bar{O}_2\log \bar{O}_2$ since along the line $\alpha=\alpha_0$, $\bar{O}_1$ and $\bar{O}_2$ do not change.
Simple calculation with Hoeffding's inequality show that with probability at least $1-\delta$ the following hold
$$\bar{O}_{11}(0)=\bar{O}_{22}(0)=\bar{O}_{12}(0)=\frac{n\lambda_n}{4}\left(\pi_1^2+\omega\pi_2^2+2\pi_1\pi_2r\right)+O(\lambda_n\sqrt{n}),$$
$$\bar{O}_{11}(c)=n\lambda_n\pi_1^2+\bar{O}(\lambda_n\sqrt{n}), \ \bar{O}_{22}(c)=n\lambda_n\omega\pi_2^2+O(\lambda_n\sqrt{n}),$$
$$\bar{O}_{12}(c)=n\lambda_n\pi_1\pi_2r+O(\lambda_n\sqrt{n}).$$
By the remark at the beginning of the proof of Lemma~\ref{Lem:DCBMLogLikProp}, we can take $\eta=n\lambda_n$, and therefore $\bar{Q}(U_{\bar{A}}(c))-\min_{\hat{\mathcal{R}}^\theta} \bar{Q}$ is of order $n\lambda_n$.
\end{proof}

\begin{proof}[Proof of Theorem~\ref{Thm:DCBMCons}]
Note that $\bar{Q}=\bar{Q}_{DC}$ does not satisfy all Assumptions ($1$)--($4$), therefore we can not apply Theorem~\ref{Thm:GenMetdCons} directly. Instead we will follow the idea of the proof of Lemma~\ref{Lem:MaxFuncEst}.

We first show that $\bar{Q}$ satisfies Assumption ($1$). For $\bar{Q}$, the functions $g_j$ in \eqref{Eq:GenFuncType} has the form $g(z)=z\log(z)$. We can assume that $z>1$ because otherwise $g(z)$ is bounded by a constant. Since $g^\prime(z)=1+\log(z)$, $g^\prime(z)$ does not grow faster than $\log(z)$, and therefore assumption ($1$) holds.

Note that by Lemma~\ref{Lem:DCBMCubeProj}, $\mbox{dist}\left(\hat{\mathcal{R}},\hat{\mathcal{R}}^\theta\right)$ is bounded by a constant; by Lemma~\ref{Lem:LipConst}, the Lipschitz constant of $\bar{Q}$ is of order $O\left(\sqrt{n}\log(n)\|\bar{A}\|\right)$. Therefore, to prove Lemma~\ref{Lem:MaxFuncEst}, and in turn Theorem~\ref{Thm:DCBMCons}, it is enough to consider $\bar{Q}$ on $\hat{\mathcal{R}}^\theta$.

Note also that $\bar{Q}$ may not be convex, therefore Assumption ($2$) may not hold. But we now show that the convexity of $\bar{Q}$ is not needed. In the proof of Lemma~\ref{Lem:MaxFuncEst}, the convexity of $f_B$ is used only at one place to show that \eqref{Ineq:BeforefEAche} implies \eqref{Ineq:fEAche}, or more specifically, that $f_{B}(y)\leq f_B(U_B(\hat{e}))$. Note that by \ref{Ineq:DisPbUcPbUclEU}, $\|y-U_{\bar{A}}(c)\|\leq 2\sqrt{n}\|U_A-U_{\bar{A}}\|$.
By Lemma~\ref{Lem:DCBMLogLikProp} part c, $\bar{Q}$ achieves maximum at $U_{\bar{A}}(c)$, a vertex of $\hat{\mathcal{R}}^\theta$; by Lemma~\ref{Lem:DCBMCubeProj}, the angle between two adjacent sides of $\hat{\mathcal{R}}^\theta$ is bounded away from zero and $\pi$. Thus, there exists $s\in\mathcal{E}_A$ such that
$\|y - U_{\bar{A}} (s)\|\leq M\sqrt{n}\|U_A-U_{\bar{A}}\|$. By Lemma~\ref{Lem:LipConst} we have $$|\bar{Q}(y)-\bar{Q}(U_{\bar{A}} (s))|\leq M n\log(n)\|\bar{A}\|\cdot\|U_A-U_{\bar{A}}\|.$$
Therefore in \eqref{Ineq:BeforefEAche} we can replace $y$ with $U_{\bar{A}} (s)$, and \eqref{Ineq:fEAche} follows by definition of $\hat{e}$.

We now check assumptions ($3$) and ($4$).
To check the assumption ($3$), we first assume that $U_{\bar{A}}= (D_\theta(\bar{u}_1,\bar{u}_2))^T$, where $\bar{u}_1$ and $\bar{u}_2$ are from Lemma~\ref{Lem:BMEigDecomp}, and $D_\theta=\mbox{diag}(\theta)$.
The first $\bar{n}_1=n\pi_1$ column vectors of $(\bar{u}_1,\bar{u}_2)^T$ are equal and we denote by $\xi_1$. The last $\bar{n}_2=n\pi_2$ column vectors of $(\bar{u}_1,\bar{u}_2)^T$ are also equal and we denote by $\xi_2$.
Then
\begin{eqnarray*}
  U_{\bar{A}}(c)-U_{\bar{A}}(e) &=& \sum_{i=1}^{\bar{n}_1}\theta_i(1-e_i)\xi_1+\sum_{i=\bar{n}_1+1}^n\theta_i(-1-e_i)\xi_2 \\
   &=& k_1\sum_{i=1}^{\bar{n}_1}\theta_i\xi_1 - k_2\sum_{i=\bar{n}_1+1}^n\theta_i\xi_2,
\end{eqnarray*}
where $k_1=\sum_{i=1}^{\bar{n}_1}(1-e_i)$, $k_2=\sum_{i=\bar{n}_1}^{n}(1+e_i)$, and $\|e-c\|^2=k_1+k_2$.
By Lemma~\ref{Lem:BMEigDecomp}, entries of  $\xi_1, \xi_2$ are of order $1/\sqrt{n}$ and the angle between $\xi_1, \xi_2$ does not depend on $n$, it follows that $\sqrt{n}\| U_{\bar{A}}(c)-U_{\bar{A}}(e)\|$ is of order $k_1+k_2$. By Lemma~\ref{Lem:DCBMEigDecomp}, it is easy to see that the argument still holds for the actual $U_{\bar{A}}$.

Assumption ($4$) follows directly from part ($e$) of Lemma~\ref{Lem:DCBMLogLikProp}.

Combining Assumptions (3), (4), and Lemma~\ref{Lem:MaxFuncEst}, we see that Theorem~\ref{Thm:GenMetdCons} holds.
Note that the conclusion of Lemma~\ref{Lem:OlivRest} still holds if we replace $\mathbb{E}[A]$ with $\bar{A}$, except the constant $M$ now also depends on $\xi$, that is  $M=M(r,\omega,\pi,\delta)>0$.
The upper bound in Theorem~\ref{Thm:GenMetdCons} is simplified by Lemma~\ref{Lem:OlivRest} and part d of Lemma~\ref{Lem:DCBMLogLikProp}.
The bound in Theorem~\ref{Thm:GenMetdCons} is simplified by \eqref{Ineq:AdjConc} of Lemma~\ref{Lem:OlivRest} and part e of Lemma~\ref{Lem:DCBMLogLikProp}:
\begin{equation*}
  \|e^*-c\|^2 \le M n\log n \left( \lambda_n^{-1/2} + \|U_A - U_{\mathbb{E}[A]}\|\right).
\end{equation*}
If $U_A$ is formed by eigenvectors of $A$ then using \eqref{Ineq:SubspConc} of Lemma~\ref{Lem:OlivRest}, we obtain
$$\|e^*-c\|^2 \leq \frac{ M n\log n}{\sqrt{\lambda_n}}.$$
The proof is complete.
\end{proof}

\subsection{Proof of results in Section~\ref{SubSec:MaxLogLikBM}}\label{AppendixB:BM}

We follow the notation introduced in the discussion before Lemma~\ref{Lem:DCBMLogLikProp}. Lemma~\ref{Lem:Formn1n2} provides the form of $n_1$ and $n_2$ as functions defined on the projection of the cube.

\begin{lemma}\label{Lem:Formn1n2}
Consider the block models and let $\mathcal{R}=U_{\mathbb{E}[A]}[-1,1]^n$.
In the coordinate system $x_e=U_{\mathbb{E}[A]}(e)$, the functions $n_1$ and $n_2$ defined by \eqref{Eq:Defn1n2} admit the forms
$$n_1=\sqrt{n}(\sqrt{n}+\vartheta^T x)/2,\ \ n_2=\sqrt{n}(\sqrt{n}-\vartheta^T x)/2,$$
where $\vartheta$ is a vector with $\|\vartheta\|<M$ for some $M>0$ not depending on $n$.
In the coordinate system $(\alpha,\beta)$, $n_1$ and $n_2$ admit the forms
$$n_1=\frac{\sqrt{n}}{2}\big[(1+\alpha)+s\beta\big],\ \ n_2 = \frac{\sqrt{n}}{2}\big[(1-\alpha) - s\beta\big],$$
where $s$ is a constant.
\end{lemma}
\begin{proof}[Proof of Lemma~\ref{Lem:Formn1n2}]
Let $U^*=(U_{\mathbb{E}[A]}^T,\frac{1}{\sqrt{n}}\mathbf{1})^T$ and $\mathcal{R}_{U^*} = U^*[-1,1]^n$.
For each $e\in[-1,1]^n$, let $z=\frac{1}{\sqrt{n}}\mathbf{1}^Te$, so that
$U^*e= \left(\begin{smallmatrix} x\\ z \end{smallmatrix}\right)$. Then
$$n_1 = \sqrt{n}(\sqrt{n}+z)/2, \  \ n_2 = \sqrt{n}(\sqrt{n}-z)/2.$$
By Lemma~\ref{Lem:BMCubeProj}, the first $\bar{n}_1$ row vectors of $U_{\mathbb{E}[A]}$ are equal, and the last $\bar{n}_2$ row vectors of $U_{\mathbb{E}[A]}$ are also equal.
Therefore $U^*$ has rank two, and $\mathcal{R}_{U^*}$ is contained in a hyperplane.
It follows that $z$ is a linear function of $x$, and in turn, a linear function of $(\alpha,\beta)$.

In the coordinate system $x$, $n_1(0)=n/2$ implies $z(0)=0$; $n_1(\mathbf{1})=n$ implies $z(x_{\mathbf{1}})=\sqrt{n}$; $n_1(c)=\bar{n}_1=n\pi_1$ implies $z(x_{c})=(2\pi_1-1)\sqrt{n}$.
Since $\|x_{\mathbf{1}}\|$ and $\|x_{c}\|$ are of order $\sqrt{n}$ by Lemma~\ref{Lem:BMEigDecomp} and Lemma~\ref{Lem:DCBMEigDecomp}, there exists a constant $M>0$ such that $z=\vartheta^T x$ for some vector $\vartheta$ with  $\|\vartheta\|<M$.

In the coordinate system $(\alpha,\beta)$, $n_1(0)=n_2(0)=n/2$ implies $z(0)=0$; $n_1(\mathbf{1})=n$ implies $z(1,0)=\sqrt{n}$; $n_1(-\mathbf{1})=0$ implies $z(-1,0)=-\sqrt{n}$.
Therefore along the line $\beta=0$, $z(\alpha,0) = \sqrt{n}\alpha$.
For any fixed $\alpha$, $z$ is a linear function of $\beta$ with the same coefficient, so $z(\alpha,\beta) = \sqrt{n}\alpha + s\sqrt{n}\beta$ for some constant $s$.
\end{proof}

Lemma~\ref{Lem:BMLogLikProp} show some properties of $\bar{Q}_{BM}$. Parts (b) gives a weaker version of convexity of $\bar{Q}_{BM}$. Part (c) together with Lemma~\ref{Lem:BMCubeProj} will be used to replace Assumption (2). Part (d) verifies Assumption (4), and part (e) simplifies the upper bound in part (d).

\begin{lemma}\label{Lem:BMLogLikProp}
Consider $\bar{Q}=\bar{Q}_{BM}$ on $\mathcal{R}=U_{\mathbb{E}[A]}[-1,1]^n$. Then
\begin{description}
  \item[($a$)] $\bar{Q}(\alpha,0)$ is a constant.
  \item[($b$)] $\frac{\partial^2 \bar{Q}} {\partial\beta^2}\geq 0$, $\frac{\partial \bar{Q}} {\partial\beta}>0$ if $\beta>0$ and $\frac{\partial \bar{Q}} {\partial\beta}<0$ if $\beta<0$.
      Thus, $\bar{Q}$ achieves minimum when $\beta=0$ and maximum on the boundary of $\mathcal{R}$.
  \item[($c$)] $\bar{Q}$ is convex on the boundary of $\mathcal{R}$. Thus, $\bar{Q}$ archive maximum at $\pm U_{\mathbb{E}[A]}c$.
  \item[($d$)] If $\bar{Q}(U_{\mathbb{E}[A]}c)-\bar{Q}(x)\leq \epsilon$ then
  $$\|U_{\mathbb{E}[A]}c-x\| \leq 4\epsilon\sqrt{n}\left(\bar{Q}(U_{\mathbb{E}[A]}c)-\min_{\mathcal{R}} \bar{Q}\right)^{-1}.$$
  \item[($e$)] $\bar{Q}(U_{\mathbb{E}[A]}(c))-\min_{\mathcal{R}} \bar{Q}$ is of order $n\lambda_n$.
\end{description}
\end{lemma}
\begin{proof}[Proof of Lemma~\ref{Lem:BMLogLikProp}]
Let $G = \bar{O}_1\log\frac{\bar{O}_1}{n_1}+ \bar{O}_2\log\frac{\bar{O}_2}{n_2}$, then $\bar{Q}_{BM}=\bar{Q}_{DCBM}+2G$.
By Lemma~\ref{Lem:DCBMLogLikProp}, to show ($a$), ($b$), and ($c$), it is enough to show that $G$ satisfies those properties.
Parts ($d$) and ($e$) follow from ($a$), ($b$), and ($c$) by the same argument used to prove Lemma~\ref{Lem:DCBMLogLikProp}.
Note that if we multiply $\bar{O}_1$ and $\bar{O}_2$ by a positive constant, or multiply $n_1$ and $n_2$ by a positive constant, then the behavior of $G$ does not change, since $\bar{O}_1+\bar{O}_2$ is a constant.
Therefore by Lemma~\ref{Lem:Formn1n2} we may assume that
\begin{eqnarray*}
\bar{O}_1 &=& 2(1+\alpha), \ \ \bar{O}_2 = 2(1-\alpha),\\
n_1 &=& (1+\alpha)+s\beta, \ \ n_2 = (1-\alpha) - s\beta.
\end{eqnarray*}

($a$) It is easy to see that $G(\alpha,0)$ is a constant.

($b$) Simple calculation shows that
$$\frac{\partial G}{\partial\beta} = \frac{4s^2\beta}{n_1n_2}, \ \ \
\frac{\partial^2 G}{\partial\beta^2} = \frac{4s^2}{(n_1n_2)^2}\left(1-\alpha^2+s^2\beta^2\right),$$
and the statement follows.

($c$) We show that $G$ is convex on the segment connecting $U_{\mathbb{E}[A]}c$ and $U_{\mathbb{E}[A]}\mathbf{1}$.
With the parametrization \eqref{SegPar}, $n_1$ and $n_2$ have the form
$$ n_1 = (1+\alpha)+s(1-\alpha), \ \ n_2 = (1-\alpha)-s(1-\alpha),$$
for some constant $s$. Simple calculation shows that
$$\frac{d^2 G}{d\alpha^2}=\frac{4}{\bar{O}_1}-\frac{2(1-s)}{n_1}-\frac{4s(1-s)}{n_1^2}.$$
Note that when $\alpha=1$, the right hand side equals $s^2\geq 0$.
Therefore, to show that $G$ is convex, it is enough to show that the second derivative of $G$ is non-increasing.
The third derivative of $G$ has the form
$$\frac{d^3 G}{d\alpha^3}=\frac{8s^2}{n_1^3(1+\alpha)^2}\big[(3\alpha+1)s-3\alpha-3\big].$$
Note that $n_1\geq 0$ implies $s\geq-\frac{1+\alpha}{1-\alpha}$; $n_2\geq 0$ implies $s\leq 1$.
Consider function $h(s)=(3\alpha+1)s-3\alpha-3$ on $\left[\frac{1+\alpha}{1-\alpha},1\right]$.
Since
$$h\left(\frac{1+\alpha}{1-\alpha}\right)=\frac{-4(1+\alpha)}{1-\alpha}\leq 0, \ \ h(1)=-2<0,$$
$h(s)\leq 0$ and $G$ is convex.
\end{proof}

Note that $\bar{Q}_{BM}$ does not have the exact form of \eqref{Eq:GenFuncType}. A small modification shows that Lemma~\ref{Lem:MaxFuncEst} still holds for $\bar{Q}_{BM}$.

\begin{lemma}\label{Lem:BMAnalLem2}
Let $Q=Q_{BM}$, $\bar{Q}=\bar{Q}_{BM}$, and $U_A$ be an approximation of $U_{\mathbb{E}[A]}$.
Under the assumptions of Theorem~\ref{Thm:BMCons}, there exists a constant $M=M(r,w,\pi,\delta)>0$ such that with probability at least $1-n^{-\delta}$,
we have
\begin{equation*}
  \bar{Q}(x_c)-\bar{Q}(x_{e^*}) \leq M n \log n \left( \sqrt{\lambda_n} + \lambda_n \|U_A-U_{\mathbb{E}[A]}\| \right).
\end{equation*}
In particular, if $U_A$ is the matrix whose row vectors are leading eigenvectors of $A$, then
\begin{equation*}
  \bar{Q}(x_c)-\bar{Q}(x_{e^*}) \leq M n\log n \sqrt{\lambda_n}.
\end{equation*}
\end{lemma}

\begin{proof}[Proof of Lemma~\ref{Lem:BMAnalLem2}]
Let $G_i=O_i\log n_i$ and $\bar{G}_i=\bar{O}_i\log n_i$ for $i = 1,2$.
Also, let $G=Q_{DCBM}$ and $\bar{G}=\bar{Q}_{DCBM}$.
Then
$$Q = G + G_1 + G_2, \quad \bar{Q} = \bar{G} + \bar{G}_1 + \bar{G}_2.$$
In the proof of Theorem~\ref{Thm:DCBMCons} we have shown that $G$ satisfies Assumption (1).
Therefore inequality \eqref{Ineq:fAandfEA} in the proof of Lemma~\ref{Lem:MaxFuncEst} also holds for $G$:
\begin{equation}\label{bound G-barG}
  |G(e) - \bar{G}(e)| \leq M n\log n \|A-\mathbb{E} A\|.
\end{equation}
The same type of inequality holds for $G_i$ as well.
Indeed, since $\|\mathbf{1}+e\|^2=2(\mathbf{1}+e)^T\mathbf{1}=4n_1$, we have
\begin{eqnarray}\label{bound G_i - bar G_i}
  |G_i(e)-\bar{G}_i(e)| &=& |\log n_1||(1+e)^T(A-\mathbb{E}[A])\mathbf{1}|\\
  \nonumber&\leq& 2n\log(n)\|A-\mathbb{E}[A]\|.
\end{eqnarray}
From \eqref{bound G-barG} and \eqref{bound G_i - bar G_i} we obtain
\begin{equation}\label{bound Q - bar Q}
  |Q(e) - \bar{Q}(e)| \leq M n\log n \|A-\mathbb{E} A\|.
\end{equation}
Let $\hat{e} = \arg\max\{\bar{Q}(e),e\in\mathcal{E}_A\}$.
Using \eqref{bound Q - bar Q} and definition of $e^*$, we have
\begin{eqnarray}\label{Ineq:BMQehatestar}
  \bar{Q}(\hat{e})-\bar{Q}(e^*)
  &\leq& \bar{Q}(\hat{e}) - Q(\hat{e}) + Q(e^*) - \bar{G}(e^*)  \\
  \nonumber&\leq& Mn\log(n)\|A-\mathbb{E}[A]\|.
\end{eqnarray}
Let $y\in \mbox{conv}(U_{\mathbb{E}[A]}\mathcal{E}_A)$ such that
$\|U_{\mathbb{E}[A]} (c)-y\|=\mbox{dist}\big(U_{\mathbb{E}[A]} (c),\mbox{conv}(U_{\mathbb{E}[A]}\mathcal{E}_A) \big)$.
Using the same argument as in the proof of Lemma~\ref{Lem:MaxFuncEst}, we obtain
\begin{equation}\label{projections are close}
  \|U_{\mathbb{E}[A]}(c)-y\|\leq 2\sqrt{n} \ \|U_A-U_{\mathbb{E}[A]}\|,
\end{equation}
and there exists a constant $M>0$ such that
\begin{eqnarray*}
  \big|\bar{O}_1(y)-\bar{O}_1(U_{\mathbb{E}[A]} (c))\big|
  &\leq& M n \|\mathbb{E}[A]\| . \|U_A-U_{\mathbb{E}[A]}\| \\
  &\leq& M n \lambda_n \|U_A-U_{\mathbb{E}[A]}\|.
\end{eqnarray*}
By Lemma~\ref{Lem:BMCubeProj}, the angle between two adjacent sides of $\mathcal{R}$ does not depend on $n$. Therefore \eqref{projections are close} implies that there exists $s\in\mathcal{E}_A$ such that
\begin{equation}\label{bound c - s}
  \|U_{\mathbb{E}[A]}(c) - U_{\mathbb{E}[A]} (s)\|\leq M\sqrt{n}\|U_A-U_{\mathbb{E}[A]}\|.
\end{equation}
Denote $x_e = U_{\mathbb{E}[A]}(e)$ for $e\in[-1,1]^n$.
By Lemma~\ref{Lem:LipConst} the Lipchitz constant of $\bar{G}$ on $U_{\mathbb{E}}[A][-1,1]^n$ is of order
$\sqrt{n}\|\mathbb{E}[A]\|\log n \leq \sqrt{n}\lambda_n\log n $.
Therefore from \eqref{bound c - s} we have
\begin{equation}\label{bound barG(c) - barG(s)}
  \bar{G}(x_c) - \bar{G} (x_s) \leq M n\lambda_n\log n \|U_A-U_{\mathbb{E}[A]}\|.
\end{equation}
We will show that the same inequality holds for $\bar{G}_i$, and thus also for $\bar{Q}$.
By triangle inequality we have
\begin{equation}\label{bound on barGic - barGis}
  \bar{G}_i(x_c)-\bar{G}_i(x_s)
  \leq |\bar{O}_i(x_s)-\bar{O}_i(x_c)||\log n_i(x_c)|+\bar{O}_i(x_s)\left|\log\frac{n_i(x_s)}{n_i(x_c)}\right|.
\end{equation}
To bound the first term on the right-hand side of \eqref{bound on barGic - barGis},
we note that by Lemma~\ref{Lem:LipConst},
the Lipchitz constant of $\bar{O}_i$ is of order $\sqrt{n}\|\mathbb{E}[A]\|\leq \lambda_n\sqrt{n}$.
Using \eqref{bound c - s} we obtain
\begin{eqnarray}\label{first term bound}
  |\bar{O}_i(x_s)-\bar{O}_i(x_c)||\log n_i(x_c)|
  &\leq& |\bar{O}_i(x_s)-\bar{O}_i(x_c)| \log n \\
  \nonumber&\leq& M n\lambda_n\log n \|U_A-U_{\mathbb{E}[A]}\|.
\end{eqnarray}
We now bound the second term on the right-hand side of \eqref{bound on barGic - barGis}.
By Lemma~\ref{Lem:Formn1n2},
there exist $M^\prime>0$ not depending on $n$ and a vector $\vartheta$ such that $\|\vartheta\|\leq M^\prime$ and
\begin{eqnarray}\label{bound on ni}
  |n_i(x_c)-n_i(x_s)| &=& |\vartheta^T(x_c-x_s)|/2\leq M^\prime \|x_c-x_s\| \\
  \nonumber &\leq& M^\prime \sqrt{n}\|U_A-U_{\mathbb{E}[A]}\|.
\end{eqnarray}
Note that $n_i(x_c)=\bar{n}_i = n\pi_1$ and $|n_i(x_c)-n_i(x_s)| = o(n)$ by \eqref{bound on ni}.
Using \eqref{bound on ni} and the inequality $\log(1+t)\leq 2|t|$ for $|t|\leq 1/2$, we have
\begin{eqnarray}\label{log bound}
  \left|\log\frac{n_i(x_s)}{n_i(x_c)}\right|
  &=& \left|\log\left( 1 + \frac{n_i(x_s)-n_i(x_c)}{n_i(x_c)} \right)\right| \\
  \nonumber&\leq& \frac{2M^\prime\sqrt{n}\|U_A-U_{\mathbb{E}[A]}\|}{n_i(x_c)}.
\end{eqnarray}
By definition, $\bar{O}_i(x_s)$ is at most $O(n\lambda_n)$. Therefore from \eqref{log bound} we obtain
\begin{equation}\label{O x log bound}
  |\bar{O}_i(x_s)|\cdot\left|\log\frac{n_i(x_s)}{n_i(x_c)}\right|
  \leq M \lambda_n \sqrt{n}\|U_A-U_{\mathbb{E}[A]}\|.
\end{equation}
Using \eqref{bound barG(c) - barG(s)}, \eqref{bound on barGic - barGis}, \eqref{first term bound},
\eqref{O x log bound}, and the fact that $\bar{Q}(x_s)\leq \bar{Q}(x_{\hat{e}})$, we get
\begin{eqnarray}\label{ bound Qc - Qhate}
  \quad \bar{Q}(x_c)-\bar{Q}(x_{\hat{e}}) \leq \bar{Q}(x_c)-\bar{Q}(x_s)
  \leq M n\lambda_n\log n \|U_A-U_{\mathbb{E}[A]}\|.
\end{eqnarray}
Finally, from \eqref{Ineq:BMQehatestar}, inequality \eqref{Ineq:AdjConc} of Lemma~\ref{Lem:OlivRest},
and \eqref{ bound Qc - Qhate}, we obtain
\begin{eqnarray*}
  \bar{Q}(x_c)-\bar{Q}(x_{e^*}) \leq M n \log n \left( \sqrt{\lambda_n} + \lambda_n \|U_A-U_{\mathbb{E}[A]}\| \right).
\end{eqnarray*}
If $U_A$ is formed by eigenvectors of $A$ then it remains to use inequality \eqref{Ineq:SubspConc}
of Lemma~\ref{Lem:OlivRest}. The proof is complete.
\end{proof}

\begin{proof}[Proof of Theorem~\ref{Thm:BMCons}]
The proof is similar to that of Theorem~\ref{Thm:DCBMCons}, with the help of Lemma~\ref{Lem:BMLogLikProp} and Lemma~\ref{Lem:BMAnalLem2}.
\end{proof}

\subsection{Proof of results in Section~\ref{Subsec:NGMod}}\label{AppendixB:NG}

We follow the notation introduced in the discussion before Lemma~\ref{Lem:DCBMLogLikProp}.

\begin{proof}[Proof of Theorem~\ref{Thm:NGCons}]
Note that $\bar{Q}=\bar{Q}_{NG}$ does not have the exact form of \eqref{Eq:GenFuncType}. We first show that $\bar{Q}$ is Lipschitz with respect to $\bar{O}_1$, $\bar{O}_2$, and $\bar{O}_{12}$, which is stronger than assumption ($1$) and ensures that the argument in the proof of Lemma~\ref{Lem:MaxFuncEst} is still valid.

To see that $\bar{Q}$ is Lipschitz, consider the function $h(x,y)=\frac{xy}{x+y}$, $x\geq 0, y\geq 0$.
The gradient of $h$ has the form $\nabla h(x,y)=\left(\frac{y^2}{(x+y)^2},\frac{x^2}{(x+y)^2}\right)$.
It is easy to see that $\nabla h(x,y)$ is bounded by $\sqrt{2}$.
Therefore $h$ is Lipschitz, and so is $\bar{Q}$.

Simple calculation shows that $\bar{Q} = 2b\beta^2$. Therefore $\bar{Q}$ is convex, and by Lemma~\ref{Lem:BMCubeProj}, it achieves maximum at the projection of the true label vector. Thus, assumption ($2$) holds.
Assumption ($3$) follows from Lemma~\ref{Lem:BMEigDecomp} by the same argument used in the proof of Theorem~\ref{Thm:DCBMCons}. Assumption ($4$) follows from the convexity of $\bar{Q}$ and the argument used in the proof of part ($e$) of Lemma~\ref{Lem:DCBMLogLikProp}. Note that $\bar{Q}(0)=0$ and $\bar{Q}(c)$ is of order $n\lambda_n$, therefore Theorem~\ref{Thm:NGCons} follows from Theorem~\ref{Thm:GenMetdCons}.
\end{proof}

\subsection{Proof of results in Section~\ref{SubSec:MaxComExtrCrn}}\label{AppendixB:EXTR}

We follow the notation introduced in the discussion before Lemma~\ref{Lem:DCBMLogLikProp}. We first show some properties of $\bar{Q}_{EX}$. Parts (b) and (c) verify Assumption (2), and part (d) verifies Assumption (4).
\begin{lemma}\label{Lem:CECritProp}
Let $\bar{Q}=\bar{Q}_{EX}$. Then
\begin{description}
  \item[($a$)] $\bar{Q}(\alpha,0)=0$.
  \item[($b$)] $\bar{Q}$ is convex.
  \item[($c$)] If $\pi_1^2>r\pi_2^2$ then the maximum value of $\bar{Q}$ is $n\lambda_n\pi_1\pi_2(1-r)$ and it is achieved at $x_c=U_{\mathbb{E}[A]}(c)$; if $\pi_1^2\leq r\pi_2^2$ then the maximum value of $\bar{Q}$ is $n\lambda_n\pi_1\pi_2r(\frac{\pi_2^2}{\pi_1^2}-1)$ and it is achieved at $x_{-c}=-U_{\mathbb{E}[A]}(c)$.
  \item[($d$)] Let $x_{\max}$ be the maximizer of $\bar{Q}$. If $\bar{Q}(x_{\max})-\bar{Q}(x)\leq \epsilon=o(n\lambda_n)$ then $\|x_{\max}-x\|\leq 2\epsilon\sqrt{n}(\bar{Q}(x_{\max}))^{-1}$.
\end{description}
\end{lemma}
\begin{proof}[Proof of Lemma~\ref{Lem:CECritProp}]
Note that multiplying $\bar{O}_{11}$, $\bar{O}_{12}$ by a positive constant, or multiplying $n_1$ and $n_2$ by a constant does not change the behavior of $\bar{Q}$. Therefore by Lemma~\ref{Lem:Formn1n2} we may assume that
\begin{eqnarray*}
  \bar{O}_{11} &=& (1+\alpha)^2 +b\beta^2, \ \  \bar{O}_{12} = (1-\alpha^2) -b\beta^2,\\
  n_1 &=& 1+\alpha +s\beta, \ \ n_2 = 1-\alpha -s\beta.
\end{eqnarray*}

($a$) It is straightforward that $\bar{Q}(\alpha,0)=0$.

($b$) Let $z = s\beta$, $r=s^2/b>0$, and $h(\alpha,z)=\frac{z^2-r(1+\alpha)z}{z+1+\alpha}$, then $\bar{Q}=\frac{2}{r}h(\alpha,z)$. Simple calculation shows that the Hessian of $h$ has the form
$$\nabla h=\frac{2(r+1)}{(z+1+\alpha)^3}
\left(
  \begin{array}{cc}
    (1+\alpha)^2 & -z(1+\alpha) \\
    -z(1+\alpha) & z^2 \\
  \end{array}
\right)
,$$
which implies that $h$ and $\bar{Q}$ are convex.

($c$) Since $\mathcal{R}=U_{\mathbb{E}[A]}[-1,1]^n$ is a parallelogram by Lemma~\ref{Lem:BMCubeProj} and $\bar{Q}$ is convex by part ($b$), it reaches maximum at one of the vertices of $\mathcal{R}$. The claim then follows from a simple calculation.

($d$) Note that $|\bar{Q}(x_c)-\bar{Q}(x_{-c})|=|\frac{\pi_2}{\pi_1}n\lambda_n(\pi_1^2-r\pi_2^2)|$ is of order $n\lambda_n$, therefore if $\bar{Q}(x_{\max})-\bar{Q}(x)\leq \epsilon=o(n\lambda_n)$ then $x_{\max}$ and $x$ belong to the same part of $\mathcal{R}$ divided by the line $\beta=0$. In other words, if $\hat{x}$ is the intersection of the line going through $x$ and $x_{\max}$ and the line $\beta=0$, then $x$ belongs to the segment connecting $x_{\max}$ and $\hat{x}$. By convexity of $\bar{Q}$ and the fact that $\bar{Q}(\hat{x})=0$ from part ($a$) and part ($b$), we get
$$\frac{\|x_{\max}-x\|}{\|x_{\max}-\hat{x}\|}
\leq\frac{\bar{Q}(x_{\max})-\bar{Q}(x)}{\bar{Q}(x_{\max})-\bar{Q}(\hat{x})}\leq\frac{\epsilon}{\bar{Q}(x_{\max})}.$$
It remains to bound $\|x_{\max}-\hat{x}\|$ by $2\sqrt{n}$.
\end{proof}

Note that $\bar{Q}_{EX}$ does not have the exact form of \eqref{Eq:GenFuncType}. The following Lemma shows that the argument used in the proof of Lemma~\ref{Lem:MaxFuncEst} holds for $\bar{Q}_{EX}$.

\begin{lemma}\label{Lem:CEAnalLem2}
Let $\bar{Q}=\bar{Q}_{EX}$ and assume that the assumption of Theorem \ref{Thm:ComExtrCons} holds.
Let $U_A$ be an approximation of $U_{\mathbb{E}[A]}$.
Then there exists a constant $M=M(r,\pi,\delta)>0$ such that with probability at least $1-n^{-\delta}$, we have
\begin{equation}\label{Ineq:CEQcQestar}
  \bar{Q}(c)-\bar{Q}(e^*)\leq M n \lambda_n\left( \lambda_n^{-1/2} + \|U_A-U_{\mathbb{E}[A]}\|\right).
\end{equation}
In particular, if $U_A$ is a matrix whose row vectors are eigenvectors of $A$, then
\begin{equation*}
  \bar{Q}(c)-\bar{Q}(e^*)\leq M n \sqrt{\lambda_n}.
\end{equation*}
\end{lemma}

\begin{proof}[Proof of Lemma~\ref{Lem:CEAnalLem2}]
Note that $\|\mathbf{1}+e\|^2=2(\mathbf{1}+e)^T\mathbf{1}=4n_1$.
Using inequality \eqref{Ineq:AdjConc} of Lemma~\ref{Lem:OlivRest}, we have
\begin{eqnarray*}
  \left|\frac{n_2}{n_1}O_{11}-\frac{n_2}{n_1}\bar{O}_{11}\right| &=&
  \left|\frac{n_2}{n_1}(\mathbf{1}+e)^T(A-\mathbb{E}[A])(\mathbf{1}+e)\right|\\
  &\leq& \frac{n_2}{n_1} \|\mathbf{1}+e\|^2\|A-\mathbb{E}[A]\|\\
  &\leq& Mn_2\sqrt{\lambda_n}\leq Mn\sqrt{\lambda_n},\\
  |O_{12}-\bar{O}_{12}|&\leq& Mn\sqrt{\lambda_n}.
\end{eqnarray*}
Therefore
\begin{equation*}\label{Ineq:CEfuncErr}
  |Q(e)-\bar{Q}(e)|\leq Mn\sqrt{\lambda_n}.
\end{equation*}
Let $\hat{e} = \arg\max\{\bar{Q}(e),e\in\mathcal{E}_A\}$.
Then $Q(e^*)\geq Q(\hat{e})$ and hence
\begin{eqnarray}\label{Ineq:CEQehatestar}
  \bar{Q}(\hat{e})-\bar{Q}(e^*) &\leq& \bar{Q}(\hat{e}) - Q(\hat{e}) + Q(e^*)- \bar{Q}(e^*) \\
  \nonumber&\leq& M n\sqrt{\lambda_n}.
\end{eqnarray}
Let $y\in \mbox{conv}(U_{\mathbb{E}[A]}\mathcal{E}_A)$ such that
$\|U_{\mathbb{E}[A]} (c)-y\|=\mbox{dist}\big(U_{\mathbb{E}[A]} (c),\mbox{conv}(U_{\mathbb{E}[A]}\mathcal{E}_A) \big)$.
By the same argument as in the proof of Lemma \ref{Lem:MaxFuncEst}, we have
\begin{equation}\label{distance of projections EXTR}
  \|U_{\mathbb{E}[A]}(c)-y\|\leq 2\sqrt{n} \ \|U_A-U_{\mathbb{E}[A]}\|.
\end{equation}
From Lemma~\ref{Lem:LipConst}, the Lipchitz constant of $\bar{O}_i$ is of order $\sqrt{n}\|\mathbb{E}[A]\|\leq \sqrt{n}\lambda_n$.
Using \eqref{distance of projections EXTR}, we get
\begin{eqnarray}\label{Ineq:CEODiff}
  \big|\bar{O}_{1i}(y)-\bar{O}_{1i}(U_{\mathbb{E}[A]} (c))\big|
  &\leq& M n \lambda_n  \|U_A-U_{\mathbb{E}[A]}\|.
\end{eqnarray}
Denote $x_e = U_{\mathbb{E}[A]}(e)$ for $e\in[-1,1]^n$.
By Lemma~\ref{Lem:Formn1n2}, there exist $M^\prime >0$ not depending on $n$
and a vector $\vartheta$ such that $\|\vartheta\|\leq M^\prime$ and for $i=1,2,$
\begin{eqnarray}\label{difference ni EXTR}
  |n_i(x_c)-n_i(y)|
  &=&     |\vartheta^T(x_c-y)|/2\leq M^\prime\|x_c-y\| \\
  \nonumber&\leq&  M^\prime\sqrt{n}\|U_A-U_{\mathbb{E}[A]}\|, \quad \mathrm{by \ \eqref{distance of projections EXTR}}.
\end{eqnarray}
Note that $n_i(x_c)=\bar{n}_i = \pi_i n$ and $|n_i(x_c)-n_i(y)| = o(n)$ by \eqref{difference ni EXTR}.
Therefore from \eqref{difference ni EXTR} we obtain
\begin{eqnarray*}
  \left|\frac{\bar{n}_2}{\bar{n}_1}-\frac{n_2(y)}{n_1(y)}\right| \leq Mn^{-1/2}\|U_A-U_{\mathbb{E}[A]}\|.
\end{eqnarray*}
Together with \eqref{Ineq:CEODiff} and the fact that $\bar{O}_{11}(y)\leq n\lambda_n$, we get
\begin{eqnarray*}
  |\bar{Q}(x_c)-\bar{Q}(y)| &\leq& \frac{\bar{n}_2}{\bar{n}_1}|\bar{O}_{11}(x_c)-\bar{O}_{11}(y)|+
  \left|\frac{\bar{n}_2}{\bar{n}_1}-\frac{n_2(y)}{n_1(y)}\right| \bar{O}_{11}(y)\\
  &+& |\bar{O}_{12}(y)-\bar{O}_{12}(x_c)| \\
  &\leq&  Mn\lambda_n \|U_A-U_{\mathbb{E}[A]}\|.
\end{eqnarray*}
The convexity of $\bar{Q}$ by Lemma~\ref{Lem:CECritProp} then imply
\begin{equation}\label{Ineq:CEQcQehat}
  \bar{Q}(x_c)-\bar{Q}(x_{\hat{e}})\leq Mn\lambda_n \|U_A-U_{\mathbb{E}[A]}\|.
\end{equation}
Finally, adding \eqref{Ineq:CEQehatestar} and \eqref{Ineq:CEQcQehat} we get \eqref{Ineq:CEQcQestar}.
If $U_A$ is formed by eigenvectors of $A$, then it remains to use inequality \eqref{Ineq:SubspConc}
of Lemma~\ref{Lem:OlivRest}. The proof is complete.
\end{proof}

\begin{proof}[Proof of Theorem~\ref{Thm:ComExtrCons}]
The proof is similar to that of Theorem~\ref{Thm:DCBMCons}, with the help of Lemma~\ref{Lem:BMEigDecomp}, Lemma~\ref{Lem:CECritProp}, and Lemma~\ref{Lem:CEAnalLem2}.
\end{proof}

\end{document}